\documentclass{article}

\usepackage[round]{natbib}

\newif\ifcondencedmath
\condencedmathtrue
\usepackage{subcaption} % in your preamble

\usepackage[preprint]{icml2026}

\usepackage[utf8]{inputenc} % allow utf-8 input
\usepackage[T1]{fontenc}    % use 8-bit T1 fonts
\usepackage{hyperref}       % hyperlinks
\usepackage{url}            % simple URL typesetting
\usepackage{booktabs}       % professional-quality tables
\usepackage{amsfonts}       % blackboard math symbols
\usepackage{nicefrac}       % compact symbols for 1/2, etc.
\usepackage{microtype}      % microtypography
\usepackage{xcolor}         % colors
\usepackage{amsmath}

\newcommand{\cU}{\mathcal{U}}
\newcommand{\cK}{\mathcal{K}}

\newcommand{\E}{\mathbb{E}}
\newcommand{\cJ}{\mathcal{J}}

\newcommand{\indep}{\perp \!\!\! \perp}

\newcommand{\cY}{\mathcal{Y}}

\newcommand{\cI}{\mathcal{I}}

\newcommand{\R}{\mathbb R}

\renewcommand{\empty}{\varnothing}

\newcommand{\eqdef}{\vcentcolon=} % changed the macro to this because the above one looses vertical space when used inline

\def\<#1,#2>{\langle #1,#2\rangle}

\usepackage{dsfont}

\renewcommand{\leq}{\leqslant}
\renewcommand{\geq}{\geqslant}

\def\<{\langle}
\def\>{\rangle}
\def\|{\Vert}

\def\var{{\rm var\,}}

\newcommand{\esp}[1]{\mathbb{E}\left[#1\right]}

\newcommand{\set}[1]{{{\left\{ #1\right\}}}} % ||1||% Operators
\newcommand{\abs}[1]{{{\left| #1\right|}}} % ||1||% Operators
\newcommand{\proba}[1]{\mathbb{P}\left(#1\right)}

\newcommand{\cA}{\mathcal{A}}

\newcommand{\cN}{\mathcal{N}}
\newcommand{\cX}{\mathcal{X}}

\newcommand{\cP}{\mathcal{P}}

\newcommand{\cD}{\mathcal{D}}
\newcommand{\one}{\mathds{1}}

\usepackage{amsfonts}
\usepackage{amsthm}
\usepackage{amsmath}
\usepackage{amssymb}
\usepackage{mathrsfs}
\usepackage{amscd}
\usepackage{caption}
\usepackage{indentfirst}
\usepackage{cleveref}

\usepackage{cleveref}
\usepackage{amssymb}
\usepackage{mathtools}

\usepackage{graphicx}
\usepackage{subfloat}
\usepackage{subcaption}
\newtheorem{theorem}{Theorem}
\newtheorem{assumption}{Assumption}
\newtheorem{definition}{Definition}
\newtheorem{lemma}{Lemma}
\newtheorem{proposition}{Proposition}

\crefname{assumption}{assumption}{assumptions}

\newtheorem{remark}{Remark}

\newcommand{\gdpmu}{\zeta} % The best choice would probably be \alpha, \beta, or \gamma

\newcommand{\Bmu}{B_\mu}
\newcommand{\Bpi}{B_\pi}

\usepackage[colorinlistoftodos,textwidth=2.1cm,disable]{todonotes}

\definecolor{mydarkgreen}{RGB}{39,130,67}
\definecolor{mydarkred}{RGB}{192,25,25}
\definecolor{mydarkblue}{RGB}{0,0,140}
\hypersetup{
	colorlinks=true,        % false: boxed links; true: colored links
	linkcolor=blue,         % color of internal links
	citecolor=blue,         % color of links to bibliography
	urlcolor=blue           % color of external links
}

\icmltitlerunning{Model Agnostic Differentially Private Causal Inference}

\date{}

\begin{document}

\twocolumn[
  \icmltitle{Model Agnostic Differentially Private Causal Inference}



  \icmlsetsymbol{equal}{*}

  \begin{icmlauthorlist}
    \icmlauthor{Christian Janos Lebeda}{equal,inria}
    \icmlauthor{Mathieu Even}{equal,inria}
    \icmlauthor{Aurélien Bellet}{inria}
    \icmlauthor{Julie Josse}{inria}
  \end{icmlauthorlist}

  \icmlaffiliation{inria}{Inria, Université de Montpellier, INSERM, France}

  \icmlcorrespondingauthor{Christian Janos Lebeda}{christian-janos.lebeda@inria.fr}
  \icmlcorrespondingauthor{Mathieu Even}{mathieu.even@inria.fr}
  %\icmlcorrespondingauthor{Firstname2 Lastname2}{first2.last2@www.uk}

  \icmlkeywords{casual inference, differential privacy}

  \vskip 0.3in
]



%\printAffiliationsAndNotice{}  % no special notice (required even if empty)
\printAffiliationsAndNotice{\icmlEqualContribution}

%\maketitle

\begin{abstract}
%\vspace{-1.2em}
\looseness=-1 Estimating causal effects from observational data is essential in fields such as medicine, economics and social sciences, where privacy concerns are paramount.
We propose a general, model-agnostic framework for differentially private estimation of average treatment effects (ATE) that avoids strong structural assumptions on the data-generating process or the models used to estimate propensity scores and conditional outcomes. In contrast to prior work, which enforces differential privacy by directly privatizing these nuisance components,
% and results in a privacy cost that scales with model complexity,
our approach decouples nuisance estimation from privacy protection. This separation allows the use of flexible, state-of-the-art black-box models, while differential privacy is achieved by perturbing only predictions and aggregation steps within a fold-splitting scheme with ensemble techniques. We instantiate the framework for three classical estimators---the G-Formula, inverse propensity weighting (IPW), and augmented IPW (AIPW)---and provide formal utility and privacy guarantees, together with privatized confidence intervals. Empirical results on synthetic and real data show that our methods maintain competitive performance under realistic privacy budgets.
% We further extend our framework to support meta-analysis of multiple private ATE estimates. Our results bridge a critical gap between causal inference and privacy-preserving data analysis.
\end{abstract}

%\vspace{-2em}

\section{Introduction}\label{section:intro}
%\section{INTRODUCTION}\label{section:intro}

Quantifying treatment effects at the population level is a critical task with far-reaching implications for economics, policy, and public health. These estimates guide health interventions, shape policy decisions, and inform clinical recommendations. In modern evidence-based medicine, randomized controlled trials (RCTs) are considered the gold standard for estimating treatment effects—such as the average treatment effect (ATE) via the risk difference (RD)—because they effectively isolate causal effects from confounding factors \citep{imbens2015causal}.
However, RCTs present several limitations: they are costly and time-consuming, and their strict inclusion/exclusion criteria often result in study samples that are small and differ significantly from the broader population eligible for treatment. In contrast, observational data---gathered without deliberate intervention, as in disease registries, cohorts, biobanks, epidemiological studies, or electronic health records---offer a compelling alternative. These data sources are typically more comprehensive, less expensive to collect, and often provide a more accurate representation of real-world populations.
Yet, estimating treatment effects from observational data is challenging due to the presence of \textit{confounders}. Most existing approaches handle this by adjusting for \textit{observed confounders} through regression of \textit{nuisance} parameters \citep{chernozhukov2018double,wager2024causal}.

\looseness=-1 The sensitive nature of individual-level data in causal inference studies raises significant privacy concerns, as medical, economic, and social science datasets often contain personal information that must be protected against reidentification and attribute reconstruction. Although releasing a single causal estimate in isolation is unlikely to pose major privacy risks, in practice such releases rarely occur alone: they are typically accompanied by descriptive statistics (e.g., means, variances, correlations) and are often repeated across multiple studies using the same dataset. This compounding effect has enabled powerful reconstruction attacks, as in the US Decennial Census, where even simple aggregates allowed adversaries to recover a large fraction of individual-level records despite protections such as swapping \citep{doi:10.1073/pnas.2218605120}. Similar risks have been observed in genomics \citep{membership_genomic} and location data \citep{SHRINGARPURE2015631,DBLP:conf/ndss/PyrgelisTC18}. These findings underscore the need for causal inference methods that provide formal privacy guarantees, especially given that large-scale resources such as UK Biobank \citep{10.1371/journal.pmed.1001779} or MIMIC \citep{JohnsonAlistairE.W.2023Mafa} are reused across thousands of studies with overlapping subpopulations.

Differential Privacy (DP) \citep{DworkMNS06,Dwork2014a} is widely recognized as the gold standard for such guarantees. By adding noise proportional to the sensitivity of the computation—i.e., the maximum change in output due to a single individual—it ensures that each individual has limited influence on the output and provides a formal mechanism to track cumulative privacy loss across repeated analyses. While individual analysts (e.g., clinicians or economists) may have little direct incentive to use DP, demand is driven by data custodians—hospitals, biobanks, and statistical agencies—who must comply with regulations such as HIPAA 
and GDPR. Institutions such as the US Census Bureau, Eurostat, and Statistics Canada have adopted or are actively evaluating DP for their data release pipelines. % , demonstrating both legal and practical incentives to limit disclosure risks.
Beyond regulatory compliance, strong privacy guarantees can also increase participant willingness to contribute data, improving dataset representativeness and reducing sampling bias.

% Separately, the field of privacy has garnered increasing attention, especially as concerns grow over the use of sensitive data. In machine learning (ML) and statistical estimation, the privacy goal is to ensure that the specific data points used to train a model or compute an estimate cannot be inferred from the output. This concern becomes especially acute in causal inference with medical data, where datasets are often reused across many studies, amplifying the risk of privacy breaches.
% Although traditional causal inference methods do not release raw data but only estimators or models, this alone does not guarantee privacy. In fact, recent work shows that such outputs can still leak information about individual data points \citep{zhu2019deep}. This motivates the need for causal inference methods that come with formal privacy guarantees, ensuring that no adversary can reconstruct sensitive personal information from the outputs.
% Differential Privacy (DP) \citep{DworkMNS06} is widely recognized as the gold standard for providing such guarantees, ensuring that individual data points have a minimal impact on the algorithm's output distribution and thus cannot be recovered.

Current methods for differentially private causal inference are however limited in scope. Some are designed for simplified settings—such as RCTs using private difference-in-means estimators \citep{ChenCBRO24,ohnishi2025locally,yao2024privacypreservingquantiletreatmenteffect}—and therefore fall short of realizing the potential of large-scale observational datasets. Others depend on strong structural assumptions, including parametric forms for nuisance parameters % and bounded covariates
\citep{lee2019privacyIPW,ohnishi2024covariateBalancing}, which often do not hold in practice and lead to biased estimators, limiting their applicability to real-world data. Crucially, these methods enforce privacy by privatizing the nuisance models themselves, which requires designing a differentially private training algorithm for the chosen model class and incurs privacy costs that scale with model complexity—costs that are unnecessary if only the final causal estimate is released.

\textbf{Contributions.}
Motivated by these limitations, we propose a flexible, non-parametric, and model-agnostic framework for differentially private causal inference, designed to meet the practical needs of the broader causal inference community. 
Our approach privately estimates the ATE from observational data without structural assumptions, avoiding bias from model misspecification. Unlike prior methods, any \emph{non-private} estimator can be used for nuisance parameters (e.g., propensity scores or outcome regressions), while privacy is enforced at the prediction and aggregation stages. The framework proceeds through four main steps: 
\textit{(i)} split the data into $K$ folds;
\textit{(ii)} estimate nuisance parameters on each fold; 
\textit{(iii)} aggregate the predictions of these $K$ nuisance models on held-out data from the $K-1$ other folds to reduce sensitivity—and thus the amount of noise required for privacy;
\textit{(iv)} construct a private ATE estimator from these predictions and calibrated Gaussian noise.
% \textit{(v)} compute a differentially private variance estimate to produce a confidence interval.
From this framework, we derive differentially private versions of three widely-used estimators: the plug-in \textit{G-Formula}, \textit{Inverse Propensity Weighting} (IPW), and \textit{Augmented Inverse Propensity Weighting} (AIPW) estimators. We provide a unified analysis of privacy and utility guarantees, leveraging the recent Gaussian DP framework \citep{Dong22GDP} to derive tight privacy bounds, while quantifying the asymptotic variance of the estimators, shedding light on their privacy-utility trade-offs. 
Importantly, we show that our general framework allows to privately construct valid confidence intervals by adapting techniques that are commonly used in practice, such as bootstrap methods. 
We also extend our framework to perform meta analysis, in which several private ATE estimates are aggregated to effectively reduce variance. Finally, we demonstrate the effectiveness of our approach through experiments on synthetic and real data, highlighting its superiority over prior methods in practically relevant scenarios.

\section{Preliminaries} \label{sec:prelim}
%\section{PRELIMINARIES} \label{sec:prelim}

% In this section, we introduce the necessary background in causal inference and differential privacy.

\subsection{Causal Inference Framework} \label{sec:prelim-causal-setup}

We assume access to a dataset of $n$ independent and identically distributed (i.i.d.) patients drawn from an underlying distribution $\cP$.
Each patient $i\in[n]$ is described by their \emph{covariates} (feature vector) $X_i\in\cX$, \emph{treatment assignment} $A_i\in\set{0,1}$ (we assume binary treatments), and \emph{observed outcome} $Y_i\in\cY\subset\R$ (treatment response).
Let $\cD\eqdef(X_i,A_i,Y_i)_{i\in[n]}\sim \cP$ denote the dataset, which may originate from either a randomized control trial (RCT) or an observational study.
We adopt the potential outcome framework, which formalizes the notion of causal effects by positing the existence of two potential outcomes for each individual $i$, $Y_i(1),Y_i(0)\in\cY$, corresponding to the outcomes under treatment and control, respectively.
Throughout this work, we make the following standard assumptions.

\begin{assumption}[Stable Unit Treatment Value]\label{asm:SUTVA} 
For all $i\in[n]$, $Y_i = A_i Y_i(1) + (1 - A_i) Y_i(0)$.
%    where $Y(1), Y(0)$ represent the potential outcomes under treatment and control, respectively.
\end{assumption}

% The fundamental challenge of causal inference lies in the fact that we can only observe one of the potential outcomes $Y_i(0)$ and $Y_i(1)$ for each data point. 
% Even though the joint distribution of $(Y_i(0),Y_i(1))$ will never be observed, we can still construct estimators of the ATE with the RD, provided that some identifiability assumption holds: there is no hidden confounder and any patient has a probability of being treated that is bounded away from $0$ and $1$.

\begin{assumption}[Unconfoundedness]
    \label{asm:unconfoundedness}
    For all $i\in[n]$,
    we have $\{Y_i(0), Y_i(1)\} \indep A_i \vert X_i$.
    % We have $Y\indep A|X$.
\end{assumption}

\begin{assumption}[Overlap]\label{asm:overlap}
	There exists $\eta\in(0,1/2]$ such that for all $x\in\mathcal X$, we have $\pi(x)\in[\eta,1-\eta]$, where $\pi(x) \eqdef\proba{A=1|X=x}$ is the propensity score.
\end{assumption}

We quantify treatment effects by estimating the average treatment effect (ATE) with the risk difference (RD).

\begin{definition}[Average Treatment Effect (ATE)]
    The ATE with the RD is defined as
    $\tau \eqdef \E [Y_i(1) - Y_i(0)]$.
    The expectation is taken over the distribution $\cP$.   
\end{definition}

\subsection{Non-Private ATE Estimators} \label{sec:prelim-ATE-estimator}

We now present standard non-private ATE estimators under \Cref{asm:SUTVA,asm:unconfoundedness,asm:overlap}, 
beginning with the associated nuisance functions.
The propensity score $\pi(x) \eqdef\proba{A=1|X=x}$
is the probability of receiving treatment as a function of the covariates, while
$\mu_a(x) \eqdef \E[Y \vert X = x, A = a]$ is the expected outcome as a function of covariates and treatment. 
The plug-in G-Formula, IPW, and AIPW estimators we consider in this work can all be written as an average $\hat\tau \eqdef \frac{1}{n}\sum_{i=1}^n \Gamma_i$ of \emph{scores}  % Plug-in G-formula, IPW, and AIPW estimators % of the ATE write as:
% then define \emph{scores}
$\Gamma_i$ for each patient $i\in[n]$ % using the covariates $(X_i,A_i,Y_i)$ and non-parametric estimates $\hat\mu_0,\hat\mu_1,\hat\pi$ of $\mu_0,\mu_1,\pi$.
using $(X_i,A_i,Y_i)$ and (non-parametric) estimators $\hat\mu_0,\hat\mu_1,\hat\pi$ of $\mu_0,\mu_1,\pi$:\footnote{These non-parametric models are usually trained on a dataset independent of $(X_i,A_i,Y_i)_{i\in[n]}$.} 
% \jj{in the footnote not clear, you spaek about similar asympotic variance but we have not yet mentioned it}
% These scores are then used to define the ATE estimators as:
% \begin{equation}\label{eq:estim_gen}
%     \hat\tau \eqdef \frac{1}{n}\sum_{i=1}^n \Gamma_i\,.
% \end{equation}
% Instances of the scores are respectively 
\ifcondencedmath
\begin{align*}\label{eq:scores-g-formula-ipw-aipw}
\text{G-Formula: }&\Gamma_i^\mathrm{G} = \hat\mu_1(X_i)-\hat \mu_0(X_i) \,,\\
\text{IPW: } &\Gamma_i^\mathrm{IPW} =  \frac{A_i}{\hat\pi(X_i)}Y_i - \frac{1-A_i}{1-\hat\pi(X_i)}Y_i \,,\\
\text{AIPW: } &\Gamma_i^\mathrm{AIPW} = \hat\mu_1(X_i)-\hat \mu_0(X_i)~+ \\
\frac{A_i}{\hat\pi(X_i)}&(Y_i-\hat\mu_1(X_i)) - \frac{1-A_i}{1-\hat\pi(X_i)}(Y_i-\hat\mu_0(X_i))\,.
\end{align*}
\else
\begin{equation}\label{eq:score-g-formula-ipw}
\text{G-Formula: }\Gamma_i^\mathrm{G} = \hat\mu_1(X_i)-\hat \mu_0(X_i) \,,\quad
\text{IPW: } \Gamma_i^\mathrm{IPW} =  \frac{A_i}{\hat\pi(X_i)}Y_i - \frac{1-A_i}{1-\hat\pi(X_i)}Y_i
\end{equation}
\begin{equation}\label{eq:score-aipw}
\text{AIPW: }\Gamma_i^\mathrm{AIPW} = \hat\mu_1(X_i)-\hat \mu_0(X_i) + \frac{A_i}{\hat\pi(X_i)}(Y_i-\hat\mu_1(X_i)) - \frac{1-A_i}{1-\hat\pi(X_i)}(Y_i-\hat\mu_0(X_i))\,.
\end{equation}
\fi
\looseness=-1 We denote by $\hat\tau_\mathrm{G}$, $\hat\tau_\mathrm{IPW}$ and $\hat\tau_\mathrm{AIPW}$ the G-Formula, IPW and AIPW estimators respectively.
We note that some methods instead estimate the conditional ATE, $\mu_1(x)-\mu_0(x)$, as a nuisance function and average it over the population, which can be done with, e.g., causal forests \citep{athey2019estimating}. Our framework also accommodates these approaches (see \Cref{remark:direct_cate}).

In causal inference, the theoretical ``utility'' of an estimator is typically reported as its asymptotic variance, rather than through some finite sample bounds.
% (as is the case for instance in stochastic optimization \citep{bottou2018optimization}).
For the three estimators above, the following asymptotic variance results (asymptotic unbiasedness and normality) hold \citep{wager2024causal} under \Cref{asm:SUTVA,asm:unconfoundedness,asm:overlap}:
\begin{equation*}
    %\sqrt{n}(\hat\tau - \esp{Y_i(1)-Y_i(0)})\to_\P\cN(0,V^\star)\,, \text{with}
    \sqrt{n}(\hat\tau - \esp{Y_i(1)-Y_i(0)})\Rightarrow\cN(0,V^\star)\,, \text{with}
\end{equation*}
\vspace{-.7cm}
\ifcondencedmath
\begin{align} \label{eq:var-g-formula}
V_\mathrm{G}^\star &\eqdef \var\!\big(\mu_1(X) - \mu_0(X)\big), \\
\label{eq:var-ipw}V_\mathrm{IPW}^\star &\eqdef 
   \E\!\left[\frac{Y(1)^2}{\pi(X)}\right] 
 + \E\!\left[\frac{Y(0)^2}{1-\pi(X)}\right] - \tau^2, \\
V_\mathrm{AIPW}^\star &\eqdef 
   \E\!\left[\tfrac{(Y(1)-\E[Y(1)|X])^2}{\pi(X)}\right] 
 + \E\!\left[\tfrac{(Y(0)-\E[Y(0)|X])^2}{1-\pi(X)}\right]  \nonumber\\
& \quad~~ + \var\big(\mu_1(X)-\mu_0(X)\big) \label{eq:var-aipw} .
\end{align}
\else
\begin{equation}\label{eq:var-g-formula-ipw}
	%V_\mathrm{G}^\star \eqdef \var\big(\mu_1(X)\big) +  \var\big(\mu_0(X)\big) \,,
	V_\mathrm{G}^\star \eqdef \var\big(\mu_1(X) - \mu_0(X)\big) \,,\quad
	V_\mathrm{IPW}^\star \eqdef %\var\Big(\frac{Y(1)}{\pi(X)}\Big) + \var\Big(\frac{Y(0)}{1-\pi(X)}\Big)\,,
 \mathbb{E}\Big[\frac{Y(1)^2}{\pi(X)}\Big] +  \mathbb{E}\Big[\frac{Y(0)^2}{1-\pi(X)}\Big] - \tau^2\,,
\end{equation}
\begin{equation}\label{eq:var-aipw}
	V_\mathrm{AIPW}^\star \eqdef \mathbb{E}\Big[\frac{(Y(1)-\esp{Y(1)|X})^2}{\pi(X)}\Big]\! + \mathbb{E}\Big[\frac{(Y(0)-\esp{Y(0)|X})^2}{1-\pi(X)}\Big] \!+ \var\left(\mu_1(X)-\mu_0(X)\right)\,,
\end{equation}
\fi
% with $V^\star$ as in \Cref{tab:variances},
provided that the nuisance parameters are pointwise consistent and
$\mathbb{E}\big[(\hat\mu_a(X_i)-\mu_a(X_i))^2\big]=o(n^{-1})$ (for G-Formula),
$\mathbb{E}\big[(\hat\pi(X_i)-\pi(X_i))^2\big]=o(n^{-1})$ (for IPW), or
$\mathbb{E}\big[(\hat\pi(X_i)-\pi(X_i))^2\big]\mathbb{E}\big[(\hat\mu_a(X_i)-\mu_a(X_i))^2\big]=o(n^{-1})$ (for AIPW).
Note that under this assumption, AIPW estimated with cross-fitting \citep{chernozhukov2018double} achieves the semi-parametric efficient variance.
%$V^\star_\mathrm{AIPW}\leq \min(V^\star_\mathrm{IPW},V^\star_\mathrm{G})$, %\cl{Not true for oracle estimators, but true for the estimators under the assumptions, because G-formula gets another term - see  \url{https://boughdiriahmed.github.io/book/Estimators_OT1.html}} 
Importantly, it is also \emph{doubly robust}, meaning it remains consistent if either the outcome model or the propensity score model is correctly specified---a key reason for its popularity.
Since $V^\star$ is usually unknown, it is often estimated as $\hat V= \frac{1}{n-1}\sum_{i=1}^n (\Gamma_i-\hat\tau)^2$ \citep{wager2024causal}.
Note that $V^*$ above represents the asymptotic variance of $\sqrt{n} (\hat \tau - \tau)$. Since $\hat \tau$ is scaled by $\sqrt{n}$ in this expression, the variance of the ATE estimate $\hat \tau$ itself can then be estimated by $\hat V / n$.

\subsection{Differential Privacy}\label{sec:prelim-DP}

% Informally, Differential Privacy (DP) \citep{DworkMNS06,Dwork2014a} guarantees that the output distribution of a (randomized) algorithm does not differ significantly when a single individual's data is modified, thereby limiting what an adversary can infer about any individuals. DP is widely considered the gold standard for privacy-preserving data analysis due to several key properties: \textit{(i)} it provides strong, mathematically rigorous guarantees even against adversaries with arbitrary side information; \textit{(ii)} it is \emph{immune to post-processing}, meaning that any function applied to the output of a differentially private mechanism cannot degrade its privacy guarantees; and \textit{(iii)} it supports precise accounting of the cumulative privacy loss through the principle of \emph{composition}, allowing multiple analyses to be performed on the same data with controlled privacy degradation.

\looseness=-1 Informally, Differential Privacy (DP) \citep{DworkMNS06,Dwork2014a} ensures that a (randomized) algorithm’s output changes little when a single individual's data is modified, limiting what an adversary can infer. DP is widely regarded as the gold standard for privacy-preserving analysis due to several properties: \textit{(i)} strong, rigorous guarantees even against adversaries with arbitrary side information; \textit{(ii)} \emph{post-processing immunity}, so any function of a DP output preserves privacy; and \textit{(iii)} \emph{composition}, enabling precise tracking of cumulative privacy loss across multiple analyses.

In this work, we express privacy guarantees using Gaussian Differential Privacy (GDP) \citep{Dong22GDP}, a framework that naturally describes the privacy loss from adding Gaussian noise and provides a tight analysis under composition. GDP is increasingly recommended for reporting DP guarantees \citep{gomez2025varepsilondeltaconsideredharmful}, and it can be converted directly to the classical $(\varepsilon, \delta)$-DP definition (see Appendix~\ref{app:gdp-to-dp}).

Intuitively, GDP frames privacy as a hypothesis testing problem. Consider an adversary trying to determine whether an algorithm $\mathcal{A}$ was run on a dataset $\cD$ or a neighboring dataset $\cD'$ that differs from $\cD$ by replacing a single individual's data.
% (sometimes called replace-one or bounded DP).
The harder it is to distinguish these datasets, the stronger the privacy.
To formalize this, we use the trade-off function, which captures the adversary's best possible performance.

\begin{definition}[Trade-off function]
Let $P$ and $Q$ be two probability distributions. For any decision rule $\phi$ that outputs the probability of rejecting $H_0 : P$ in favor of $H_1 : Q$, the trade-off function is
\[
T(P,Q)(\alpha) \eqdef \inf \big\{1 - \mathbb{E}_Q[\phi] : \mathbb{E}_P[\phi] \leq \alpha \big\},
\]
\looseness=-1 where the infimum is over all possible rules $\phi$, and $0 \leq \alpha \leq 1$ is the allowable type I error (false positive rate).
\end{definition}

Now we can define GDP.

\begin{definition}[$\gdpmu$-Gaussian Differential Privacy \citep{Dong22GDP}]
    \label{def:gdp}
    Let $\gdpmu\geq 0$. We say that a randomized algorithm $\cA$ % \colon \cX^n \rightarrow \mathcal{R}$
    satisfies $\gdpmu$-GDP if, for all $\cD \sim \cD'$, we have
    \[
        T(\cA(\cD),\cA(\cD')) \geq T(\mathcal{N}(0, 1), \mathcal{N}(\gdpmu, 1)) \enspace .
    \]
\end{definition}

In other words, a randomized algorithm $\mathcal{A}$ satisfies $\gdpmu$-Gaussian Differential Privacy (GDP) if distinguishing its outputs on two neighboring datasets (differing by one individual) is at least as hard as distinguishing between two unit-variance Gaussian distributions whose means differ by $\gdpmu$.
Thus, the smaller $\gdpmu$, the stronger the privacy guarantees. 
Our private ATE estimates rely on the Gaussian mechanism, a simple and widely used method to enforce $\gdpmu$-GDP.
% Note that the privacy parameter for GDP it typically denoted by $\mu$. We break convention and use $\gdpmu$ %\cl{This is a latex command "gdpmu". Use it whenever we discuss the privacy parameter, so we can easily change it later} 
% since we use $\mu$ for one of the nuisance functions that we introduce in the next subsection.
The Gaussian mechanism adds noise proportional to the \emph{sensitivity} of the function—how much the output could change if a single individual’s data is modified.

\begin{lemma}[Gaussian mechanism, \citealt{Dong22GDP}]
    \label{lem:gaussian-mech-dp}
    Let $f \colon \cX^n \rightarrow \mathbb{R}$ be a function with sensitivity $\max_{\cD \sim \cD'} \vert f(\cD) - f(\cD') \vert \leq \Delta$.
    Then $\cA(\cD)\eqdef f(\cD) + Z$, with $Z \sim \cN(0, \Delta^2/\gdpmu^2)$, satisfies $\gdpmu$-GDP.
\end{lemma}

%The Gaussian mechanism adds noise proportional to the \emph{sensitivity} of the function—how much the output could change if a single individual’s data is modified.

% Note that the privacy guarantees of our technique follows directly from a reduction to the Gaussian mechanism. 

\section{Related Work} \label{sec:related}
%\section{RELATED WORK} \label{sec:related}

Previous work on private ATE estimation falls into two main categories. The first relies on simple difference-in-means estimators that avoid estimating nuisance functions.
Most of these methods are limited to RCTs \citep{ChenCBRO24,javanmard2024causalinferencedifferentiallyprivate,yao2024privacypreservingquantiletreatmenteffect}, whereas our focus is on observational data. \citet{KogaCP24} propose a matching-based estimator for observational settings, pairing each data point with another that has identical covariates but opposite treatment assignment. However, their method assumes discrete covariates and requires well-balanced treatment groups (because unmatched samples are discarded). We do not make such strong assumptions in this work. 

The second category uses estimators that rely on nuisance functions like those considered in this paper.
These methods follow a two-step procedure: (i) fit the nuisance functions using a differentially private algorithm, and (ii) plug them into the ATE estimator and add noise to privatize the result. Differential privacy follows by composition from the privacy of each step.
% \cite{lee2019privacyIPW} introduced a private version of IPW. 
% They use a logistic regression estimator for the propensity score which they fit using Differentially Private Empirical Risk Minimisation (DP-ERM).
% \cite{ohnishi2024covariateBalancing} also considered IPW using a logistic regression estimator. They instead use the 2-Norm Gradient Mechanism to fit the propensity score.
% Additionally, they apply covariate balancing and show that this can significantly reduce the bias of \cite{lee2019privacyIPW} for some settings.
\cite{lee2019privacyIPW} proposed a private version of IPW using logistic regression fitted via Differentially Private Empirical Risk Minimization (DP-ERM). \cite{ohnishi2024covariateBalancing} also consider logistic regression together with covariate balancing, which can reduce the bias of \cite{lee2019privacyIPW} in some cases. Privacy is enforced with the 2-Norm Gradient Mechanism, reducing noise but making the approach intractable with more than a few features.
Both works assume a linear propensity score and are limited to IPW, which is known to exhibit high variability in practice even in well-specified, non-private settings \citep{Kang2007}.
Furthermore, making the propensity score model DP causes the privacy cost to grow with the dimension \citep{Bassily2014a}\footnote{This dependence is hidden by assuming $\|X_i\|_2 \leq 1$ in \citep{lee2019privacyIPW,ohnishi2024covariateBalancing}.}. This cost is unnecessary if only the final ATE estimate is needed.
In contrast, our method adds noise solely to the ATE, supports arbitrary models—including non-parametric ones—and privatizes not only IPW but also the G-Formula and doubly-robust AIPW estimators, which are %more \jj{i will remove the more} 
widely used and effective in practice.
% Both \cite{lee2019privacyIPW} and \cite{ohnishi2024covariateBalancing} require that the covariates are bounded such that $\|X_i\|_2 \leq 1$. This assumption can be enforced by rescaling covariates if we have a bound on the norm. Otherwise, the covariates are clipped before applying their techniques.
% Since both mechanisms rely on IPW, they generally work well when the propensity score follows a well-specified logistic regression model.
% But they do not have the double-robustness that makes AIPW popular.\ab{TODO: compress and highlight the main arguments: structural assumptions / simple logistic models; privatize models overkill + suffer with dimension (cite lower bounds) / limited to IPW, known to be numerically unstable / does not scale with dimension for Lee}
% \ab{(in fact I think the bounded covariate thing hides the effect of dimension in the utility in these previous work, so maybe worth mentioning in this regard)}

Concurrently with our work, \citet{DBLP:journals/csda/GuhaR25} proposed private IPW estimators using a data-splitting scheme that bears some resemblance to ours. However, their approach differs fundamentally: they estimate the propensity scores and ATE \emph{independently on each data split}, then aggregate the results in a meta-analysis fashion. In contrast, we estimate nuisance functions on all but one split and compute the ATE on the held-out split—reducing variance and ensuring the independence structure needed for establishing asymptotic normality in non-parametric settings. Additionally, their method is limited to IPW with binary outcomes.

%\clin{Discuss subsample and aggregate either here or in remark of next section}

Other existing work on DP causal inference has considered settings with public treatment assignments and outcomes \citep{agarwalSingh21}, as well as Conditional Average Treatment Effect (CATE) estimation \citep{nori21DPEBM,NiuNQCNK22,schroder2025differentially}. Here, we focus on the ATE rather than CATE. The 
ATE is the primary estimand in many applied settings.

\section{Differentially Private ATE Estimators}
%\section{DIFFERENTIALLY PRIVATE ESTIMATORS}
%\section{PRIVATE ATE ESTIMATORS}
\label{sec:DP_estimators}
In this section, we introduce our private ATE estimators.
% $\esp{Y_i(1)-Y_i(0)}$, the average treatment effect with the risk difference.
We first present the general framework, then instantiate it to obtain differentially private G-Formula, IPW, and AIPW estimators, provide a unified privacy and utility analysis, and show how to construct private confidence intervals.\looseness=-1
% We provide privacy analyses of these estimators in a unified analysis, as well as utility analyses in the form of privately computed confidence intervals.

\vspace{-0.5em}
\subsection{General Framework}
\label{sec:framework}

In the ATE estimators defined in \Cref{sec:prelim-ATE-estimator}, it is necessary to protect both the training data used to learn the nuisance functions and the data points on which the ATE is evaluated. Our goal in this work is to avoid making strong structural assumptions (such as parametric forms) about the data or the nuisance functions. However, this flexibility introduces a key challenge: without such assumptions, the sensitivity of these estimators---which determines the scale of the noise required for DP (see \Cref{lem:gaussian-mech-dp})---becomes inherently high. In particular, modifying a single data point can arbitrarily impact the fitted nuisance models, leading to a sensitivity of $O(1)$ that does not decrease with dataset size, making accurate private ATE estimation infeasible even with large datasets.
% Under the minimal assumptions adopted in this work—namely, standard causal inference assumptions and boundedness of the models—we observe that the sensitivity \jj{in the DP background, sensitivity is the concept which should be carefully described more than why Gaussian DP instead of something else?}of these estimators is inherently high. Specifically, modifying a single data point can affect the entire set of nuisance models, resulting in a sensitivity of order 1, which does not diminish as the dataset size increases.
%
A straightforward strategy would be to train \textit{differentially private non-parametric models} for the nuisance functions $\pi, \mu_0, \mu_1$. However, this tends to be privacy-inefficient and restricts the use of powerful, state-of-the-art methods such as causal forests, limiting both flexibility and performance.
% However, this approach is typically expensive in terms of privacy budget and significantly limits the class of nuisance models that can be employed, excluding efficient and state-of-the-art methods such as causal forests.

Instead, we leverage the fact that only the final ATE estimate is released. Consequently, the nuisance models do not need to be differentially private; it suffices to privatize the \emph{predictions} of these models with respect to both their training data and the data used for ATE estimation. 
Building on this insight, we partition the dataset $\cD$ into $K$ folds and learn \emph{non-private} nuisance estimators independently on each fold.
% \footnote{We discuss how to set the parameter $K$ in \Cref{app:exp:effect-k}.}.
For each fold, we aggregate predictions from the models trained on the $K-1$ other folds to form ensemble estimates. This folding strategy substantially reduces the overall sensitivity of the final computation, allowing us to use flexible, non-private nuisance models while still providing strong privacy guarantees for the ATE.

Formally, our framework follows the following five steps.
\vspace{-.9em}
\begin{enumerate}
	\item \textbf{Data splitting.}
	Partition $\cD$ into $K$ folds $(\cI_k)_{k\in[K]}$. For each patient $i\in[n]$, there exists a unique $k(i)\in[K]$ such that $i\in\cI_{k(i)}$.
	\item \textbf{Nuisance estimation.} Learn $K$ sets of (non-parametric) estimators $\big(\hat \pi^{(k)},\hat\mu_1^{(k)},\hat\mu_0^{(k)}\big)_{k\in[K]}$, where $\hat \pi^{(k)},\hat\mu_1^{(k)},\hat\mu_0^{(k)}$ are trained on $\cI_k$.
	\item \textbf{Aggregation.} For all $i\in[n]$, compute the estimated nuisance parameters of patient $i$ as: %, in order to reduce the sensitivity of $\mu, \frac{1}{\pi}$ and $\frac{1}{1-\pi}$ used in the ATE estimators:
		\begin{align}
		    \hat\pi_1(i) = & \textstyle\Big( \frac{1}{K-1}\sum_{k\in[K]\setminus\set{k(i)}} \frac{1}{\hat\pi^{(k)}(X_i)}\Big)^{-1}\,,\nonumber\\
		    1-\hat\pi_0(i) = &\textstyle\Big( \frac{1}{K-1}\sum_{k\in[K]\setminus\set{k(i)}} \frac{1}{1-\hat\pi^{(k)}(X_i)}\Big)^{-1}\,,\nonumber\\
		    \hat\mu_a(i) = & \textstyle\frac{1}{K-1}\sum_{k\in[K]\setminus\set{k(i)}} \hat\mu_a^{(k)}(X_i) \,.\label{eq:models_nuisances}
		    %1-\tilde\pi(i) = &\left( \frac{1}{K-1}\sum_{k\in[K]\setminus\set{k(i)}} \frac{1}{1-\hat\pi^{(k)}(X_i)}\right)^{-1}\,.
		\end{align}
	    In other words, the propensity score (resp. the conditional outcome) of patient $i$ is estimated using a harmonic (resp. arithmetic) mean over the $K-1$ models trained \emph{excluding} patient $i$.\footnote{We use the harmonic mean for propensity scores as the ATE estimators involve the \emph{inverses} of these scores.}
	    Intuitively, this design ensures that the sensitivity of the nuisance estimators $\hat\mu_a,\frac{1}{\hat\pi_1},\frac{1}{1-\hat\pi_0}$ is $O(1/K)$, which decreases as the number of folds $K$ increase.
	    Note that this aggregation implies separate propensity estimates for treated and control patients. 
     % \jj{when reading this I am wondering what I do with that, as it is not so usual, so what are the consequences?}
     % \me{it is true that we do not have $\hat\pi_1+\hat\pi_0=1$, but I don't know if we you give more details ? also, they are never used together, so I guess it doesn't matter... and also, if all propensity models on all folds give the same result, then we don't have this 'problem'}
    \item \textbf{Private ATE estimator.}
    For all $i\in[n]$, compute the score $\hat \Gamma_i$ of data point $i$ as in \Cref{sec:prelim-ATE-estimator} but with the ``ensemble'' nuisance models defined in \Cref{eq:models_nuisances}.
    For instance, $\hat \Gamma_i=\hat\mu_1(i)-\hat \mu_0(i) + \frac{A_i}{\hat\pi_1(i)}(Y_i-\hat\mu_1(i)) - \frac{1-A_i}{1-\hat\pi_0(i)}(Y_i-\hat\mu_0(i))$ for AIPW.
% $\hat \Gamma_i = \hat\mu_1(X_i)-\hat\mu_0(X_i)$, $\hat\Gamma_i = \frac{A_iY_i}{\hat\pi_1(X_i)}- \frac{(1-A_i)Y_i}{1-\tilde\pi_0(X_i)}$ and $\hat \Gamma_i=\hat\mu_1(X_i)-\hat \mu_0(X_i) + \frac{A_i}{\hat\pi_1(X_i)}(Y_i-\hat\mu_1(X_i)) - \frac{1-A_i}{1-\tilde\pi_0(X_i)}(Y_i-\hat\mu_0(X_i))$
% for the G-Formula, IPW and AIPW estimators
%     \begin{equation*}
%     \hat\Gamma_i = \Psi\left(i,(X_i,A_i,Y_i),\hat\pi,\hat\mu_0,\hat\mu_1,\tilde\pi\right)\,,
% \end{equation*}
% where the form of the function $\Psi$ depends the choice of ATE estimator (see \Cref{sec:prelim-ATE-estimator}); 
%where the function $\Psi$ is defined similarly to \Cref{eq:aipw,eq:ipw,eq:g-formula}.
%For example, when constructing a private version of $\hat\tau_\mathrm{IPW}$ we have $\hat\Gamma_i = \frac{A_i}{\hat\pi(X_i)}Y_i + \frac{1-A_i}{1-\tilde\pi(X_i)}Y_i$.
Then, aggregate these scores and add Gaussian noise to get the private ATE estimate:
% Our differentially private estimate, that we denote as , in its general form writes as:
%in order to build our differentially private estimate, that we denote as $\hat\tau_\mathrm{DP}$, and that writes as, in its general form:
\ifcondencedmath
\begin{equation}
    \hat\tau_\mathrm{DP}\eqdef \hat\tau + \cN(0,\sigma_1^2)\,, \text{ with } \textstyle\hat\tau \eqdef \frac{1}{n}\sum_{i=1}^n \hat \Gamma_i\,.
\end{equation}
\else
\begin{equation}
    \hat\tau_\mathrm{DP}\eqdef \hat\tau + \cN(0,\sigma_1^2)\,,\quad \text{with}\qquad\textstyle\hat\tau \eqdef \frac{1}{n}\sum_{i=1}^n \hat \Gamma_i\,.
\end{equation}
\fi
% where $\hat\tau$ is the non-privatized estimator:
% \begin{equation*}
%    \hat\tau \eqdef \frac{1}{n}\sum_{i=1}^n \hat \Gamma_i\,. 
% \end{equation*}
% This formulation encompasses the three estimators we introduced in \Cref{sec:prelim-ATE-estimator}: G-formula, IPW, and AIPW.
% We present the form of each estimator separately below.

    \item \textbf{%(Optional) \jj{do we need to put optional?} \cl{No, but we don't use it for the experiments now, which might be confusing. I removed it for now.}
    Private variance estimation.}
    Privately estimate the variance of the scaled estimator as
\ifcondencedmath
\begin{equation}\label{eq:variance}
    \hat V_\mathrm{DP}\eqdef \left( \sqrt{\frac{1}{n-1}\sum_{i=1}^n\big(\hat\Gamma_i-\hat\tau\big)^2} + \cN(0,\sigma_2^2)\right)^2 + n\sigma_1^2\,.
\end{equation}
% \begin{multline}\label{eq:variance}
%     \textstyle\hat V_\mathrm{DP}\eqdef \Big( \sqrt{\frac{1}{n(n-1)}\sum_{i=1}^n\big(\hat\Gamma_i-\hat\tau\big)^2} + \cN(0,\sigma_2^2)\Big)^2 \\ +\sigma_1^2+2.33\sigma_2^2\,.
% \end{multline}
\else
\begin{equation}\label{eq:variance}
    \textstyle\hat V_\mathrm{DP}\eqdef \Big( \sqrt{\frac{1}{n-1}\sum_{i=1}^n\big(\hat\Gamma_i-\hat\tau\big)^2} + \cN(0,\sigma_2^2)\Big)^2+\sigma_1^2+2.33\sigma_2^2\,.
\end{equation}
\fi
\end{enumerate}
%\cl{We could remove the addition of $2.33\sigma_2^2$ here now that we don't return a confidence interval. Adding $2.33\sigma_2^2$ is confusing without the explanation. This should be updated in Theorem 2 as well.}

%\vspace{-0.5em}
Different choices of scores $\hat \Gamma_i$ yield our three private ATE estimators: $\hat\tau_\mathrm{DP-G}$, $\hat\tau_\mathrm{DP-IPW}$, and $\hat\tau_\mathrm{DP-AIPW}$.
% Finally, note that Laplace noise could be used instead of Gaussian noise, leading to pure DP results.
% \ab{I feel it is not sufficient for non-DP people to understand what is at stake here. Either we develop a bit more, explaining concisely what is pure DP and why we favor Gaussian noise (better composition / if multiple estimates are released / and cleaner asymptotic behavior that combines well with standard causal inference results), or we drop it.}
% \cl{Based on AB's comment, I support removing the sentence for now. It is an easy argument if reviewers criticize the use of Gaussian noise.}

\begin{remark}[Data splitting and connections to existing techniques]
Our folding strategy relates to cross-fitting, commonly used to reduce overfitting in nuisance estimation \citep{chernozhukov2018double,athey2021policy}; $K=2$ folds recovers standard cross-fitting, while $K>2$ departs from the classical setup. Unlike traditional cross-fitting, our main motivation is to reduce sensitivity and improve privacy-utility trade-offs rather than purely statistical efficiency \citep{bach2024hyperparameter}.
Our folding strategy also resembles subsample-and-aggregate~\citep{NissimRS07,smith11}, where data is split, estimators are trained per fold, and aggregated in a private manner. Our novelty is that nuisance estimators are combined across folds, leveraging the structure of ATE estimation to effectively increase sample size without extra privacy cost (see~\Cref{app:subsample-and-aggregate} for an empirical comparison).
Finally, our scheme shares a conceptual link with PATE \citep{DBLP:conf/iclr/PapernotSMRTE18}, in that each data point affects only one model, but the goals, settings, and aggregation mechanisms differ substantially.
% Our approach also shares structural similarities with PATE (Private Aggregation of Teacher Ensembles) \citep{DBLP:conf/iclr/PapernotSMRTE18}, a differential privacy framework that splits data into disjoint subsets, each used to train a separate ``teacher'' model. These teachers vote on unlabeled examples, and their aggregated (noisy) votes are used to label a public dataset, which then trains a final "student" model. As in our framework, PATE ensures that individual data points influence only a single model. However, the goals, settings, and aggregation mechanisms differ substantially.
% Our approach also shares similarities with the subsample and aggregate framework~\citep{smith11}. The advantage of our technique compared to subsample and aggregate is that nuisance functions estimates are reused for all other folds. We discuss similarities and difference in \Cref{app:subsample-and-aggregate}.
\end{remark}

\subsection{Unified Privacy and Utility Analysis}

We now present a unified privacy and utility analysis of our private ATE estimators, deriving expressions for setting the noise multipliers $\sigma_1^2$ and $\sigma_2^2$ to achieve a target privacy level, and providing utility guarantees in terms of the resulting asymptotic variance.
% Our results are then summarized in \Cref{tab:sigma_dp}.\ab{I don't think Table 1 is very useful. We can just read off the formulas from Thm 1}

\textbf{Privacy analysis.} A well-known negative result in DP states that privately estimating the mean of unbounded reals is impossible without distributional assumptions \citep{BunNSV15}. Therefore, for our privacy analysis, we assume that both the outcomes and the nuisance function estimators are bounded.
% In practice, this boundedness assumption can be enforced through clipping.
\begin{assumption}[Bounded outcomes and nuisance estimators]\label{asm:boundedness}
    There exists ${\Bmu > 0}$ and ${\Bpi > 1}$ s.t. for all $k\in[K]$, $x\in\cX$, and $a\in\set{0,1}$ we have $\cY\subseteq[-\Bmu,\Bmu]$, $|\hat \mu^{(k)}_a(x)|\leq \Bmu$ and $\max\left( \frac{1}{\hat\pi^{(k)}(x)},\frac{1}{1-\hat\pi^{(k)}(x)}\right)\leq \Bpi$.
\end{assumption}

%\cl{Some reviewers were concerned about clipping or how to set parameters $\Bmu, \Bpi$ if the true values are unknown. We already discussed it here, perhaps the discussion can be improved.}
%\ab{I updated a bit the paragraph}
Boundedness is a standard requirement in differentially private algorithms. In many applications, outcomes are naturally bounded (e.g., binary events or blood pressure), so this natural bound can be directly used as $\Bmu$. When natural bounds are not available, they can be enforced through simple clipping—a common practice in DP algorithms. The same principle applies to propensity scores, where clipping scores near $0$ or $1$ is a well-established technique in non-private causal inference for handling poor overlap \citep{sturmer2021propensity}.
% , effectively restricting the analysis to a slightly biased subpopulation.
Although boundedness constraints for outcome models are less common in non-DP causal proofs—which typically rely on the assumptions of \Cref{thm:utility_general} to ensure well-specified models—they are sometimes imposed in specific settings \citep{BoughdiriJosseScornet2025}.
% For outcomes and outcome models, clipping is not standard in causal inference. Such boundedness constraints are less common in non-DP causal proofs, which typically rely on the assumptions of \Cref{thm:utility_general} to guarantee well-specified models, though they are occasionally imposed in specific settings \citep{BoughdiriJosseScornet2025}.
% We note that such boundedness constraints are typically absent from non-DP causal proofs,\ab{JJ comment: nuance; find some refs that use boundedness? and perhaps remove the next sentence} which rely on the assumptions of \Cref{thm:utility_general} to guarantee well-specified models. Under these assumptions, clipping is either unnecessary or equivalent to implicitly restricting analysis to (i) a subpopulation with sufficient overlap or (ii) bounded outcomes, e.g. by replacing $Y_i$ with $\min(Y_i, \Bmu)$.

We are now ready to state privacy guarantees.
\begin{theorem}[Privacy analysis]\label{thm:privacy_unified}
Assume \Cref{asm:boundedness} holds. Let $\gdpmu_1,\gdpmu_2>0$. Set the noise multiplier $\sigma_1^2$ to
	\begin{equation}\label{eq:DP-sigma1}
		\sigma_1^2 = \frac{C}{\gdpmu_1^2}\Big( \frac{1}{n} + \frac{1}{K-1} \Big)^2\,,
	\end{equation}
	where
% 	\begin{equation}\label{eq:constant_C}
% 	    C = \left\{\begin{aligned}
% 	        & \quad 16\Bmu^2 &\quad \text{if}\quad & \hat\tau_\mathrm{DP} = \hat\tau_\mathrm{DP-G},\\
% 	        & \quad 4\Bmu^2\Bpi^2 &\quad \text{if}\quad & \hat\tau_\mathrm{DP} = \hat\tau_\mathrm{DP-IPW},\\
% 	        & 16\Bmu^2\left(1+\Bpi\right)^2 &\quad \text{if} \quad&\hat\tau_\mathrm{DP} = \hat\tau_\mathrm{DP-AIPW},
% 	    \end{aligned}
% 	    \right.
% 	\end{equation}
	respectively for the cases where $\hat\tau_\mathrm{DP} = \hat\tau_\mathrm{DP-G}$, $\hat\tau_\mathrm{DP} = \hat\tau_\mathrm{DP-IPW}$ and $\hat\tau_\mathrm{DP} = \hat\tau_\mathrm{DP-AIPW}$:
	\ifcondencedmath
    \begin{equation}\label{eq:constant_C}
    \begin{aligned}
    C_\mathrm{DP\text{-}G} &= 16\Bmu^2, \quad 
    C_\mathrm{DP\text{-}IPW} = 4\Bmu^2 \Bpi^2, \\
    C_\mathrm{DP\text{-}AIPW} &= 16\Bmu^2 (1+\Bpi)^2.
    \end{aligned}
    \end{equation}
    \else
    \begin{equation}\label{eq:constant_C}
	    C_\mathrm{DP-G} =  16\Bmu^2 \,,\quad
	        C_\mathrm{DP\text{-}IPW} = 4\Bmu^2\Bpi^2 \,,\quad
	        \text{and}
	        \quad
	        C_\mathrm{DP\text{-}AIPW} =16\Bmu^2\left(1+\Bpi\right)^2 \,.
	\end{equation}
    \fi
	Then, releasing the private ATE estimator $\hat\tau_\mathrm{DP}$ satisfies $\gdpmu_1$-GDP. 
	Furthermore, if we set $\sigma_2^2$ to: %$(\alpha,\eps_1+\eps_2)-$RDP where:
    \begin{equation*}
        \sigma_2^2= \frac{2Cn}{\gdpmu_2^2(n-1)}\Big(\frac{1}{n}+\frac{1}{K-1}+ \sqrt{\frac{1}{n}+\frac{1}{K-1}}  \Big)^2 \,,
    \end{equation*}
    then releasing $\hat\tau_\mathrm{DP}$ and the variance estimate defined in \Cref{eq:variance} satisfies $\sqrt{\gdpmu_1^2 + \gdpmu_2^2}$-GDP.
\end{theorem}

The privacy guarantees hold due to the bounded sensitivity of the folding technique. The proofs are in the appendix and follow \Cref{prop:sensitivity_unified} and \Cref{lem:g-formula,lem:ipw,lem:aipw_means}.
A few remarks about this theorem are in order. First, \Cref{eq:DP-sigma1} makes explicit how the privacy noise level $\sigma_1^2$ depends on the number of data points $n$, the number of folds $K$, and the bounds $\Bmu$ and $\Bpi$ on the outcomes and imposed overlap, through the constant $C$.
Our folding and ensembling scheme ensures that the privacy cost (i.e., the variance $\sigma_1^2$ of the added noise) of achieving a given privacy level $\gdpmu_1$ for the ATE estimate decreases with both $n$ and $K$. Since in practice $K$ is typically much smaller than $n$ to preserve enough data in each fold for reliable nuisance estimation, the noise variance is of order $\frac{1}{K^2}$.
% Any desired GDP privacy level $\gdpmu_1$ can thus be reached by either adding more noise (larger $\sigma_1^2$) or by adding more folds (larger $K$).
% \jj{Is it possible maybe in the experimental section to reexplain the interpretation of $\gdpmu_1$: let say, you want a privacy of blabla which means that ... maybe to try to explain in words defintion 3, as in practice I need to understand pricesely what does it mean to be able to see if I am happy with a privacy of $\gdpmu_1$?} 
% \cl{We could add a short note, but this is in general a difficult question. There are many ways to adjust the privacy parameters. In practice you might ask the inverse question - how good privacy can we get while still getting our minimum level of utility}
Second, the additional noise $\sigma_2^2$ needed for the variance estimate is much smaller, %scaling as $\frac{1}{nK^2}$, 
so this variance can be released with only a minor additional privacy cost.
%\cl{$\hat V$ updated}
% We thus overall expect that most of the impact of adding privacy occurs in the $\frac{1}{K^2}$ term in $\sigma_1^2$.
Importantly, unlike prior work enforcing privacy by privatizing nuisance model estimators \citep{lee2019privacyIPW,ohnishi2024covariateBalancing}, our guarantees do not require any structural assumption on these models beyond boundedness.

\looseness=-1 The constant $C$ in \Cref{eq:constant_C} is the only component of the privacy bound that depends on the specific choice of ATE estimator.
This unified analysis underscores the flexibility of our approach.
In settings with large overlap $\eta$ (small $\Bpi$)---resembling RCTs---the privacy cost introduced by the use of propensity models in IPW and AIPW estimators remains modest.
However, in small-overlap regimes, which are challenging even in non-private causal inference, the required noise for privacy increases substantially for these estimators. Therefore, the G-Formula estimator, whose privacy cost is unaffected by overlap, may be preferable in such scenarios.
% \jj{this is always a bit strange and interesting, even without DP, it is more difficult to see where G formula is affected by lack of overlap, but still lack of overlap has an impact in gformula except when you are well specified.}.

% Importantly, and as opposed to previous methods that were based on private estimation of nuisance models directly, our privacy guarantees do not require any assumption on the model nuisance used (other than boundedness), and thus do not depend on the underlying dimension of the features or of the nuisance models.

\textbf{Utility analysis.} We now analyze the utility of our private ATE estimators, measured by their variance, which captures the privacy-utility trade-off.

\begin{theorem}[Utility analysis]
\label{thm:utility_general}
	Assume that \Cref{asm:SUTVA,asm:unconfoundedness,asm:overlap,asm:boundedness} hold.
	Suppose the (non-parametric) estimators $\hat\mu_1^{(k)}, \hat\mu_0^{(k)}, \hat\pi^{(k)}$ are all point-wise consistent almost everywhere and the estimators satisfy 
    $\esp{\left(\hat\mu_a(i)-\mu_a(X_i)\right)^2}=o(n^{-1})$ if $\hat\tau_\mathrm{DP}=\hat\tau_\mathrm{DP-G}$, $\esp{\left(\hat\pi(i)-\pi(X_i)\right)^2}=o(n^{-1})$ 
    if $\hat\tau_\mathrm{DP}=\hat\tau_\mathrm{DP-IPW}$, and $\esp{\left(\hat\pi(i)-\pi(X_i)\right)^2}\esp{\left(\hat\mu_a(i)-\mu_a(X_i)\right)^2}=o(n^{-1})$ if $\hat\tau_\mathrm{DP}=\hat\tau_\mathrm{DP-AIPW}$. 
    %$\esp{\left(\hat\mu_a(i)-\mu_a(X_i)\right)^2}=o_\P(n^{-1})$ if $\hat\tau_\mathrm{DP}=\hat\tau_\mathrm{DP-G}$, $\esp{\left(\hat\pi(i)-\pi(X_i)\right)^2}=o_\P(n^{-1})$ if $\hat\tau_\mathrm{DP}=\hat\tau_\mathrm{DP-IPW}$ and $\esp{\left(\hat\pi(i)-\pi(X_i)\right)^2}\esp{\left(\hat\mu_a(i)-\mu_a(X_i)\right)^2}=o_\P(n^{-1})$ if $\hat\tau_\mathrm{DP}=\hat\tau_\mathrm{DP-AIPW}$.
	Then, $\hat\tau_\mathrm{DP-G}$, $\hat\tau_\mathrm{DP-IPW}$ and $\hat\tau_\mathrm{DP-AIPW}$ are asymptotically unbiased and, for $\sigma_1^2,\sigma_2^2$ and $C$ as in \Cref{thm:privacy_unified}, we have asymptotic variance
	\ifcondencedmath
    \begin{multline*}
	    \E\left[\hat V_\mathrm{DP}\right] = V^\star + \frac{C}{\gdpmu_1^2}\left(\frac{n}{(K-1)^2} + \frac{2}{(K-1)} \right) + \\ \frac{2C}{\gdpmu_2^2}\Big(\frac{1}{K-1}+\sqrt{\frac{1}{K-1}}\Big)^2 + o\left(1\right)\,,
	\end{multline*}
 %    \begin{multline*}
	%     \hat V_\mathrm{DP}= \frac{V^\star}{n} + \frac{C}{\gdpmu_1^2}\left(\frac{1}{(K-1)^2} + \frac{2}{n(K-1)} \right) + \\ \frac{2C}{\gdpmu_2^2}\frac{1}{n}\Big(\frac{1}{K-1}+\sqrt{\frac{1}{K-1}}\Big)^2 + o\left(n^{-1}\right)\,,
	% \end{multline*}
    \else
    \begin{equation*}
	    \hat V_\mathrm{DP}= \frac{V^\star}{n} + \frac{C}{\gdpmu_1^2}\frac{1}{(K-1)^2} + \frac{4.66C}{\gdpmu_2^2}\frac{1}{n-1}\left(\frac{1}{K-1}+\sqrt{\frac{1}{K-1}}\right)^2 + o\left(n^{-1}\right)\,,
	\end{equation*}
    \fi
	where $V^\star$ denotes $V^\star_\mathrm{G}$, $V^\star_\mathrm{IPW}$ or $V^\star_\mathrm{AIPW}$ when $\hat\tau_\mathrm{DP}$ corresponds to $\hat\tau_\mathrm{DP-G}$, $\hat\tau_\mathrm{DP-IPW}$, or $\hat\tau_\mathrm{DP-AIPW}$, respectively,
	with $V^\star_\mathrm{G},V^\star_\mathrm{IPW},V^\star_\mathrm{AIPW}$ defined in \ifcondencedmath\Cref{eq:var-g-formula,eq:var-ipw,eq:var-aipw}\else\Cref{eq:var-g-formula-ipw,eq:var-aipw}\fi.
\end{theorem}

We provide the proof of the theorem in \Cref{app:unified_analysis_assumptotics}. 
This utility theorem shows that when $n$ is large, and $K$ is chosen on the order of $\sqrt n$, the variance overhead from privacy remains of the same order as the oracle variance of the non-private estimator. Unlike previous approaches that pay a cost for privatizing high-dimensional nuisance models, our result does not depend on the covariate dimension or bounds on the covariates: since we privatize only scalar quantities, the error does not scale with dimension.

\Cref{thm:utility_general} relies on convergence rate assumptions analogous to those in the non-private setting~\citep[see e.g. Equation~3.6 in][]{wager2024causal}. 
In the non-private setting, nuisance estimators are typically trained on folds containing half of the dataset using cross-fitting. In contrast, we use a larger number of smaller folds, which can affect the convergence rate. Nevertheless, by aggregating more nuisance estimates, we often mitigate this effect, as discussed in \Cref{app:exp:effect-k}. 

It is also worth noting that, in some cases, the assumptions of \Cref{thm:utility_general} do not hold even in the non-private setting. To increase applicability, confidence intervals are often reported using more flexible methods, such as the bootstrap or infinitesimal jackknife, rather than by relying on asymptotic-variance-based results of the type given by \Cref{thm:utility_general}. In \Cref{sec:private_CIs}, we present alternative methods for constructing private confidence intervals, demonstrating that our framework can readily accommodate these flexible techniques from the non-private setting.
%from the non-private literature.

%\cl{We could add note here about the convergence rate and refer to appendix for more details.}
%\ab{yes, let's think more about what we can say, and include something here (of course, we can point to the appendix for a more detailed discussion}

\begin{remark}[Model sensitivity] 
Our privacy analysis assumes \textit{worst-case} model sensitivities, allowing a single data point in $\cI_k$ to affect the output of $\big(\hat \pi^{(k)},\hat\mu_1^{(k)},\hat\mu_0^{(k)}\big)$ arbitrarily (within the bounds of Assumption~\ref{asm:boundedness}). While this is conservative, it ensures broad applicability. Nonetheless, our analysis can directly incorporate smaller sensitivities—potentially as low as $O(1/|\cI_k|)$—when using models and training algorithms for which such bounds are known \citep[see e.g.,][]{10.1162/153244302760200704,chaudhuri2011}, thereby enhancing the privacy-utility trade-off.
\end{remark}

\begin{remark}\label{remark:direct_cate}
Our approach and analysis also extend to estimators targeting the CATE function $\mu_1(x)-\mu_0(x)$, such as causal forests~\citep{athey2019estimating}. Within our framework, these can be treated—in a somewhat ad hoc way—as a special case of the G-Formula by setting $\hat\mu_0 = 0$ and $\hat\mu_1$ to the CATE estimator.
\end{remark}

\subsection{Private Confidence Intervals}
\label{sec:private_CIs}

Reporting confidence intervals (CIs) is crucial in causal inference, especially for high-stakes decisions that require quantified uncertainty. A straightforward approach is to use the asymptotic normality of $\hat \tau_\mathrm{DP}$, which holds under the assumptions of \Cref{thm:utility_general}, yielding a private, asymptotically valid $1-\alpha$ coverage CI as in \Cref{eq:private_CI_p} with variance computed privately (see \Cref{app:CI_var}). However, such intervals may be overconfident even in the non-private setting, since the assumptions of \Cref{thm:utility_general} might not fully hold. In practice, CIs in causal inference are often obtained via bootstrap methods \citep{efron1992bootstrap,efron1997improvements,davison1997bootstrap}, or through built-in error estimation, as in the causal forest (grf) package \citep{athey2019estimating}, which combines pointwise asymptotic normality of the forest-based conditional average treatment effect (CATE) estimator with variance estimates from the infinitesimal jackknife.

To construct valid confidence intervals, we estimate the CATE at a point $X_i$ using the score $\hat\Gamma_i$, and then form a local coverage interval of the form $[\hat\Gamma_{i,-},\hat\Gamma_{i,+}]$. This can be done either by estimating the asymptotic variance $\hat V:\cX \to \R$, under which $\hat\Gamma_i \approx \cN(\mathrm{CATE}(X_i), \hat V(X_i))$, or via a bootstrap procedure. We rely on the same folds as in \Cref{sec:DP_estimators} and proceed as follows.
\begin{enumerate}
    \item \textbf{Lower- and upper-bound estimation.}
    Learn $K$ sets of lower and upper-CATEs functions $(\hat\Gamma_{-}^{(k)},\hat\Gamma_{+}^{(k)})_{k\in[K]}$, where $\hat\Gamma_{\pm}^{(k)}$ are trained on $\cI_k$.
    \item \textbf{Aggregate.} For each patient $i\in[n]$, compute $\hat\Gamma_{i,\pm} = \textstyle\frac{1}{K-1}\sum_{k\in[K]\setminus\set{k(i)}} \hat\Gamma_{\pm}^{(k)}(i)$.
    \item \textbf{Privatize.} Output the confidence interval $[\hat\tau_-,\hat \tau_+]$ where 
    % $\hat\tau_{\pm}=\frac{1}{n}\sum_{i=1}^n\hat\Gamma_{i,\pm} + \cN(0,\sigma_1^2)\mp 1.96(\sigma_1)$
    $\hat\tau_{\pm}\!=\!\frac{1}{n}\sum_{i=1}^n\!\hat\Gamma_{i,\pm}\! +\! \cN(0,\sigma_1^2)\pm c(\sigma_1+\frac{\Bmu}{2\sqrt n})$.
\end{enumerate}
%\me{added the $\Bmu$ here, reversed the $\mp$ into $\pm $, and added the gaussian quantile}
We propose two strategies to learn $\hat\Gamma_{\pm}^{(k)}$ in step 1. We give an overview below and defer details to \Cref{app:details_CIs}.

\textbf{Bootstrap.}
Our first strategy applies bootstrap within each fold.
Let $\hat\Gamma_{i}^{(k)}$ denote the score for patient $i$ obtained from fold $\cI_k$. For a given number $R$ of bootstrap replications, we compute $\hat\Gamma_{i,b}^{(k)}$ for $b \in [R]$ by re-estimating the score on bootstrap samples drawn from $\cI_k$. The interval bounds $\hat\Gamma_{i,-}^{(k)}$ and $\hat\Gamma_{i,+}^{(k)}$ are then defined as the $\alpha/2$ and $1-\alpha/2$ quantiles of the debiased bootstrap distribution $(\hat\Gamma_{i,b}^{(k)}-\mathrm{median}(\hat\Gamma_{i,r}^{(k)})_r + \hat\Gamma_{i}^{(k)})_{b\in[R]}$.

\textbf{Pointwise variance.}
The second strategy, inspired by causal forests and ideally suited to them, relies on an estimate $\hat V^{(k)}_i$ of the asymptotic variance of $\hat\Gamma_{i}^{(k)}$. We construct the interval using $\hat\Gamma_{i,\pm}^{(k)} = \hat\Gamma_{i}^{(k)}\pm \Phi^{-1}(1-\alpha/2) \sqrt{\hat V^{(k)}}$, where $\Phi^{-1}$ denotes the inverse Gaussian cdf. This approach mimics the internal variance-based procedure of causal forests, but in a privacy-preserving way. Its validity relies on the pointwise asymptotic normality of the scores $\hat\Gamma_{i}^{(k)}$ around the true CATE at $X_i$.

\begin{theorem}\label{thm:CIs}
    Assume that $\hat\Gamma_{i,\pm}^{(k)}$ are bounded by $2B_\mu$ and that $\sigma_1$ is as in \Cref{eq:DP-sigma1} for $C=C_\mathrm{DP-G}$. Then $[\hat\tau_-,\hat \tau_+]$ is $\sqrt{2}\gdpmu$-GDP.
    Furthermore, if each $[\hat\Gamma_{\pm}^{(k)}(i)]$ are $1-\alpha/(nK)$ valid coverage intervals for the true CATE computed at $X_i$, then $[\hat\tau_\pm]$ is a valid $1-\alpha-\beta$ coverage interval for the ATE, if $c$ is the $1-\frac\beta2$ Gaussian quantile.
\end{theorem}

\looseness=-1 The proof of the theorem is in \Cref{app:proof_theorem_CIs}. To ensure privacy in \Cref{thm:CIs}, it is enough to clip the scores to the outcome range, which preserves statistical validity. Coverage validity holds asymptotically for both procedures: under standard bootstrap assumptions \citep{efron1992bootstrap} for the first method, and when causal forests are used for the second.

\subsection{Meta-Analysis of DP ATEs}

Meta-analyses are statistical methods for combining the estimates from $N\geq2$ independent studies to increase statistical power \citep{hunter2004methods,borenstein2021introduction}.
We extend this to \emph{private} ATE estimates.
Due to space limitations, these results can be found in \Cref{app:meta-analyses}.

\section{Experiments} \label{sec:experiments}
%\section{EXPERIMENTS} \label{sec:experiments}

\subsection{Synthetic Data}

\ifcondencedmath % Latex located here to avoid figured pushed into the reference list, we might have to play around with this by the end before submission
\begin{figure*}[!t]
  \centering
  \includegraphics[width=\textwidth]{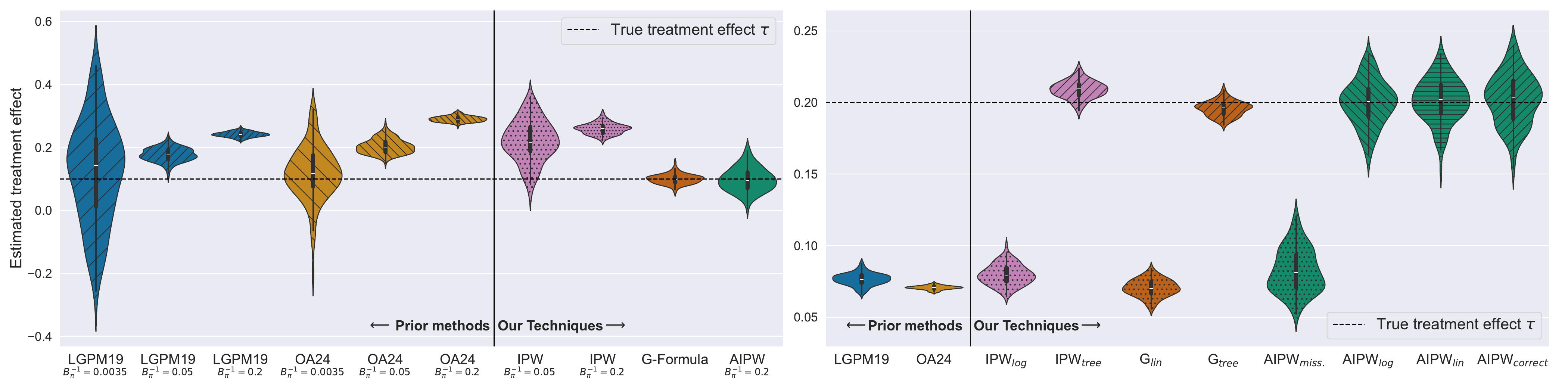}
  \caption{\textit{Left:} well-specified setting with low overlap. \textit{Right:} misspecified estimators have large bias.}
  \label{fig:experiments-plot}
\end{figure*}
\fi

We estimate the ATE on synthetic data and compare our private estimators with prior approaches \citep{lee2019privacyIPW, ohnishi2024covariateBalancing}, denoted LPGM19 and OA24, respectively. For binary outcomes (\Cref{app:exp:results}), we also compare with \cite{DBLP:journals/csda/GuhaR25} (GR25).
Covariates are sampled as $X_i \sim \cN(0,I_d)$ with $d$ up to $10$, since OA24 becomes intractable for larger $d$ or low overlap. To satisfy the bounded-norm assumption of LPGM19 and OA24, we rescale covariates so that $\|X_i\|_2 \leq 1$ for 99$\%$ of samples, clipping the rest. This gives a best-case setup for their methods. We set $\Bmu=1$, repeat each setup 100 times, and use privacy budget $1.5$-GDP ($\equiv (7.05,10^{-5})$-DP) unless otherwise noted. Further details are given in the appendix.

\textbf{Well-specified setting.}
We first study a well-specified case: logistic regression for $\pi$ and linear regression for $\mu_0,\mu_1$. IPW methods are known to suffer from high variance under limited overlap \citep{Kang2007}. To illustrate, we generate data with $\eta \approx 0.004$ (\Cref{fig:experiments-plot}, $K=200$). With $\Bpi=1/0.0035$, all IPW-based estimators show large variance; stronger clipping ($\Bpi=1/0.05$ or $1/0.2$) reduces variance but adds bias. In contrast, our G-Formula estimator is unaffected by overlap, and AIPW remains unbiased with clipping due to its outcome regression component. Additional results with larger overlap and binary outcomes (\Cref{app:exp:results}) confirm that LPGM19, OA24, and GR25 perform well, while our estimators remain competitive.

\ifcondencedmath
\else
\begin{figure}[htbp]
  \centering
   \includegraphics[width=\linewidth]{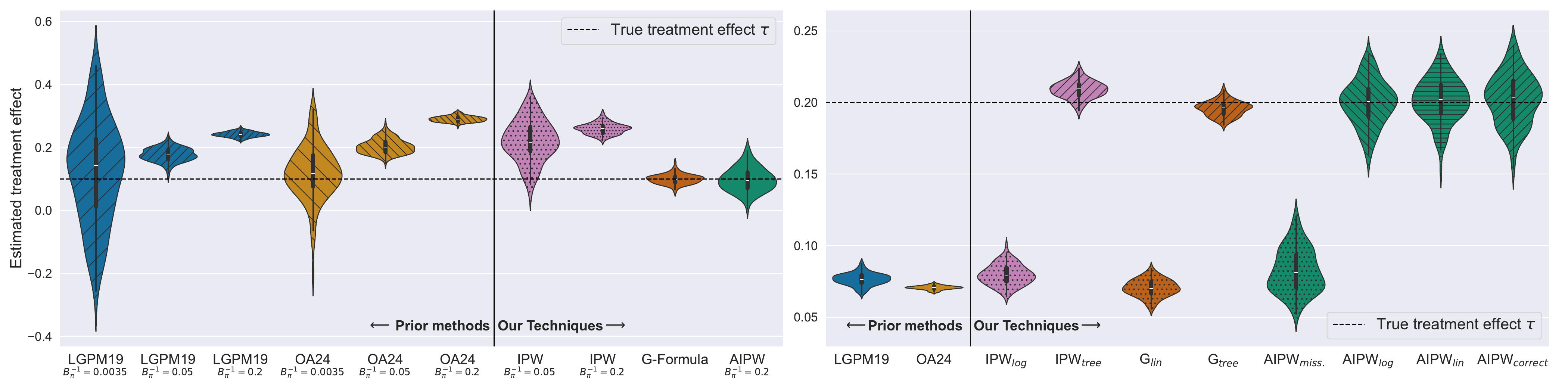}
  \caption{The left plot depicts a well-specified setting with low overlap. The right plot shows a setting where misspecified estimators have large bias.}
  \label{fig:experiments-plot}
\end{figure}
\fi

% \textbf{Misspecified setting.}
% We now turn to the setting where estimators can be misspecified. Here, propensity scores and outcome responses are modeled by decision trees. We use $K = 500$ for our estimators.
% The second plot of \Cref{fig:experiments-plot} shows that previous methods, which rely on parametric (linear) nuisance models, are thus highly biased. In contrast, when instantiated with decision tree models, our three estimators do not suffer from such biases.
% We also see that, thanks to its double-robustness property, 
% our AIPW estimator performs well even if one of the models is misspecified.

% \textbf{Effect of parameter $K$.} Our approach requires selecting the number of folds $K$. In Appendix~\ref{app:exp:results}, we analyze the effect of this parameter and examine the sensitivity of our methods to its choice.

\textbf{Misspecified setting.}
We next consider misspecification, modeling both $\pi$ and outcomes with decision trees ($K=500$). As shown in \Cref{fig:experiments-plot} (Right), prior methods relying on parametric models are highly biased. In contrast, our estimators with tree models avoid such bias. Moreover, AIPW performs well even if one model is misspecified, thanks to its double robustness property.

\textbf{Price of privacy.} 
One can see the price of privacy as requiring a larger sample size to reach the same level of evidence as in the non-private setting. \Cref{fig:sample-AZH} (left) illustrates this for the private G-Formula compared to its non-private counterpart in an experiment with binary outcomes. For a fixed sample size, the variance of the private estimator is higher. Increasing the sample size allows the private estimator to recover the same level of evidence as the non-private version.
% Further experimental details are provided in \Cref{app:exp:setup}.
% Since our technique is model agnostic, we can derive similar conclusions to the non-private equivalent estimator. 
% In order to match the error of non-private setting we typically require a larger sampling size. 
% \Cref{fig:non-private-vs-private-g-formula} compares our logistic regression G-Formula with the corresponding non-private estimator in an experiment with binary outcomes. 
% As expected, the variance of our technique is larger than the non-private estimate for matching sample sizes, but we achieve similar variance by increasing the sample size.
% Further details are in \Cref{app:exp:setup}.

\textbf{Effect of $K$.}
\looseness=-1 Our approach requires choosing the number of folds $K$. Appendix~\ref{app:exp:effect-k} analyzes its effect and shows that our method is quite stable with respect to $K$.

% \begin{figure}[h]
%   \centering
%     \begin{minipage}{0.5\columnwidth}
%         \includegraphics[width=1\columnwidth]{images/plot7-tight.pdf}
%     \end{minipage}
%     \hfill
%     \begin{minipage}{0.45\columnwidth}
%           \caption{Comparison of non-private G-formula with sample size $10000$ with our 1.5-GDP technique with sample sizes $10000$, $15000$, and $20000$. We set $K = n/100$.}
%   \label{fig:non-private-vs-private-g-formulat}
%   \end{minipage}
% \end{figure}

%\vspace{-1em}

\subsection{Real Data}

We present an experiment on a real-world dataset.
% We merged the CRASH-2 \citep{roberts2013crash}, CRASH-3 \citep{dewan2012crash} and Traumabase datasets, preprocessed them as in \citet{colnet2024causal}, ending up with 46k patients as an observational dataset. The treatment effect of TXA being non-significant, we instead generated synthetic binary outcomes, with a baseline effect of $0.3$ and an ATE of $0.22$.
We combined the NACC \citep{beekly2007national} and ADNI \citep{petersen2010alzheimer} datasets, both focused on Alzheimer's disease. After applying inclusion-exclusion criteria, the final cohort has 3.2k patients. Treatment is Acetylcholinesterase inhibitors (ChEI), and the outcome is the difference of MMSE score between baseline and 1 year after baseline. Dataset descriptions and experimental details are in \Cref{app:details_real_semi-synth}.
We evaluate two estimators: (i) the G-Formula (\Cref{sec:DP_estimators}) using regression forests to model the outcomes, and (ii) causal forests, which directly estimate the CATE, as outlined in \Cref{remark:direct_cate}. \Cref{fig:sample-AZH} (right) displays results for the non-private versions of these estimators (using $2$-fold cross-fitting and no Gaussian noise) alongside our private versions, which satisfy $1$-GDP using $K=32$ folds. 95\% confidence intervals (CIs) are computed via bootstrapping ($R=200$) for the G-Formula and using the built-in variance function for causal forests. 
We do not report (A)IPW estimators due to extremely low overlap in this dataset ($\eta < 0.05$), causing the private CIs to become unstable. Despite wider confidence intervals, the differentially private estimators still show significant results, favoring placebo over ChEI treatment.

% \begin{figure}[h]
%   \centering
%   % \subcaptionbox{Semi-synthetic traumatology data.\label{fig:trauma}}{
%   %    \includegraphics[width=0.95\columnwidth,trim=15 100 25 80,clip]{images/plot_trauma_sparse.pdf}}
  
%   % \vspace{0.3cm}
  
%     % \includegraphics[width=0.95\columnwidth,trim=40 40 25 220,clip]{images/plot_AZH_sparse.pdf}
%     \begin{minipage}{0.55\columnwidth}
%         % \includegraphics[width=1\columnwidth,trim=0 0 150 0,clip]{images/plot_AZH_sparse_2.pdf}
%         \includegraphics[width=1\columnwidth]{images/plot_AZH_sparse_3.pdf}
%     \end{minipage}
%     \hfill
%     \begin{minipage}{0.4\columnwidth}
%           \caption{Comparison of non-private and 1-GDP ATE estimators on Alzheimer dataset.}
%   \label{fig:AZH}
%   \label{fig:real_and_semi_synth}
%     \end{minipage}
% \end{figure}
\vspace{-1em}
\begin{figure}[h]
  \centering
    \begin{minipage}{0.53\columnwidth}
          \includegraphics[width=1\columnwidth]{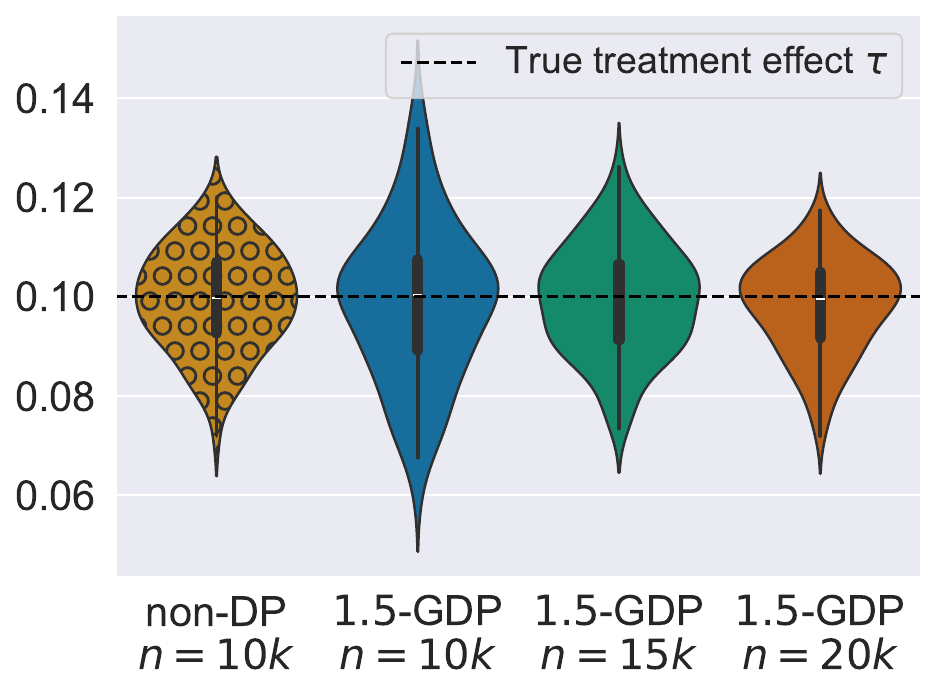}
    \end{minipage}
    \hfill
    \begin{minipage}{0.45\columnwidth}
        \includegraphics[width=1\columnwidth]{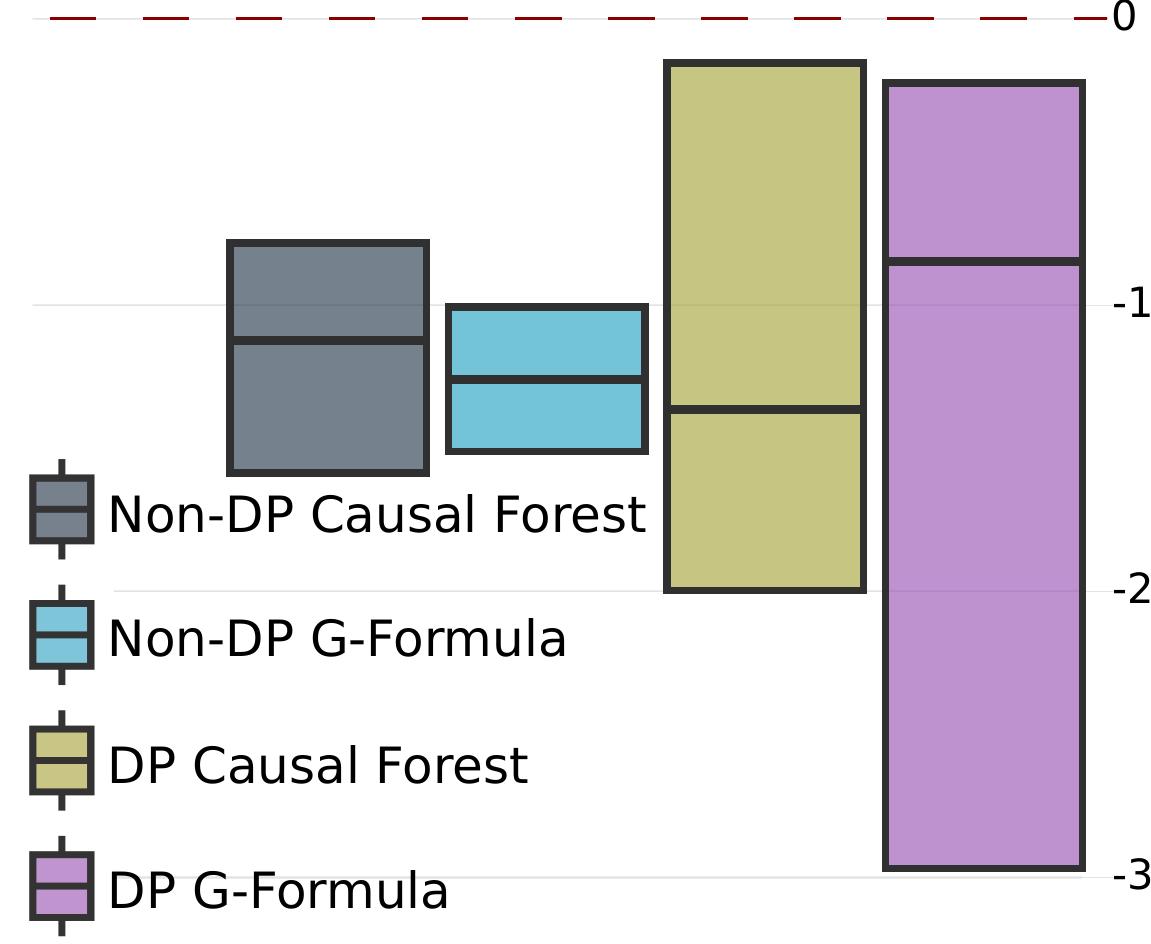}
    \end{minipage}
    \caption{\textit{Left:} Comparison of non-private G-formula with $n=10000$ with our private version for increasing $n$ and $K = n/100$. \textit{Right:} Comparison of non-private and 1-GDP ATE estimators on Alzheimer data.}
    \label{fig:sample-AZH}
\end{figure}

\vspace{-2em}
\section{Conclusion} \label{sec:conclusion}
%\section{CONCLUSION} \label{sec:conclusion}

In this work, we propose a general, \textit{plug-in and model-agnostic} approach for estimating average treatment effects under differential privacy. Unlike prior methods that rely on RCTs or parametric assumptions for nuisance models, our approach is flexible and performs well across diverse settings, as demonstrated by our experiments. This work not only supports broader adoption of differential privacy in causal inference but also suggests several promising directions for future research, including \textit{(i)} extending the framework to private estimation of heterogeneous treatment effects, and \textit{(ii)} developing methods for private generalization of causal estimates across populations, such as enabling transportability of ATEs under differing source and target populations. % going beyond worst-case sensitivities of the models by instead studying the intra-model sensitivity of nuisance estimates. As highlighted in \Cref{remark:nuisances}, this latter point is however orthogonal to this work. 

\section*{Acknowledgements}
We thank Ahmed Boughdiri and Tudor Cebere for helpful comments.
We thank Yuki Ohnishi for sharing the implementation of their technique.
The work of Christian Janos Lebeda and Aurélien Bellet is supported by grant ANR-20-CE23-0015 (Project PRIDE) and the ANR 22-PECY-0002 IPOP (Interdisciplinary Project on Privacy) project of the Cybersecurity PEPR.
Mathieu Even and Julie Josse acknowledge fundings by Theremia.

The NACC database is funded by NIA/NIH Grant U24 AG072122. NACC data are contributed by the NIA-funded ADRCs: P30 AG062429 (PI James Brewer, MD, PhD), P30 AG066468 (PI Oscar Lopez, MD), P30 AG062421 (PI Bradley Hyman, MD, PhD), P30 AG066509 (PI Thomas Grabowski, MD), P30 AG066514 (PI Mary Sano, PhD), P30 AG066530 (PI Helena Chui, MD), P30 AG066507 (PI Marilyn Albert, PhD), P30 AG066444 (PI David Holtzman, MD), P30 AG066518 (PI Lisa Silbert, MD, MCR), P30 AG066512 (PI Thomas Wisniewski, MD), P30 AG066462 (PI Scott Small, MD), P30 AG072979 (PI David Wolk, MD), P30 AG072972 (PI Charles DeCarli, MD), P30 AG072976 (PI Andrew Saykin, PsyD), P30 AG072975 (PI Julie A. Schneider, MD, MS), P30 AG072978 (PI Ann McKee, MD), P30 AG072977 (PI Robert Vassar, PhD), P30 AG066519 (PI Frank LaFerla, PhD), P30 AG062677 (PI Ronald Petersen, MD, PhD), P30 AG079280 (PI Jessica Langbaum, PhD), P30 AG062422 (PI Gil Rabinovici, MD), P30 AG066511 (PI Allan Levey, MD, PhD), P30 AG072946 (PI Linda Van Eldik, PhD), P30 AG062715 (PI Sanjay Asthana, MD, FRCP), P30 AG072973 (PI Russell Swerdlow, MD), P30 AG066506 (PI Glenn Smith, PhD, ABPP), P30 AG066508 (PI Stephen Strittmatter, MD, PhD), P30 AG066515 (PI Victor Henderson, MD, MS), P30 AG072947 (PI Suzanne Craft, PhD), P30 AG072931 (PI Henry Paulson, MD, PhD), P30 AG066546 (PI Sudha Seshadri, MD), P30 AG086401 (PI Erik Roberson, MD, PhD), P30 AG086404 (PI Gary Rosenberg, MD), P20 AG068082 (PI Angela Jefferson, PhD), P30 AG072958 (PI Heather Whitson, MD), P30 AG072959 (PI James Leverenz, MD).

Data collection and sharing for the Alzheimer's Disease Neuroimaging Initiative (ADNI) is funded by the National Institute on Aging (National Institutes of Health GrantU19AG024904). The grantee organization is the Northern California Institute for Research and Education. In the past, ADNI has also received funding from the National Institute of Biomedical Imaging and Bioengineering, the Canadian Institutes of Health Research, and private sector contributions through the Foundation for the National Institutes of Health(FNIH) including generous contributions from the following: AbbVie, Alzheimer’s Association; Alzheimer’s Drug Discovery Foundation; Araclon Biotech; BioClinica, Inc.; Biogen; Bristol-Myers Squibb Company; CereSpir, Inc.; Cogstate; Eisai Inc.; Elan Pharmaceuticals, Inc.; EliLilly and Company; EuroImmun; F. Hoffmann-La Roche Ltd and its affiliated company Genentech, Inc.; Fujirebio; GE Healthcare; IXICO Ltd.; Janssen Alzheimer Immunotherapy Research \& Development, LLC.; Johnson \& Johnson Pharmaceutical Research \& Development LLC.; Lumosity; Lundbeck; Merck \& Co., Inc.; Meso Scale Diagnostics, LLC.; NeuroRx Research; Neurotrack Technologies; Novartis Pharmaceuticals Corporation; PfizerInc.; Piramal Imaging; Servier; Takeda Pharmaceutical Company; and Transition Therapeutics.
%\fi

%%%%%%%%%%%%%%%%%%%%%%%%%%%%%%%%%%%%%%%%%%%%%%%%%%%%%%%%%%%%
\newpage
\bibliography{biblio.bib}
\bibliographystyle{plainnat}

%AISTATS
%\input{checklistAISTATS}

%\newpage
\newpage
\appendix

%AISTATS
\onecolumn

% \clin{We should avoid adding new sections/subsections if possible. If we add them, we should check that references are still consistent with the paper submission}
%\clin{I split the TODOs into smaller tasks for each section, so they are easier to take care of one-by-one.}
% \mein{TODOS:

% 1) Check notations

% 2) Proof of confidence intervals result

% 3) Polish AZH experimental section

% 4) License information for AZH datasets}

\section{Conversion from $\gdpmu$-GDP to $(\varepsilon,\delta)$-DP and GDP composition}
\label{app:gdp-to-dp}
For completeness, we discuss here the relationship between $\gdpmu$-GDP and the classical $(\varepsilon,\delta)$-DP definition, which we recall below. For further discussion on its interpretation, we refer the reader to \citep{Dwork2014a}.

\begin{definition}[$(\varepsilon, \delta)$-DP \citep{Dwork2014a}]
A randomized algorithm $\cA$ is $(\varepsilon, \delta)$-DP if for any pair of neighboring datasets $\cD\sim \cD'$, and any event $E$,
\[
    \mathbb{P}[\cA(\cD) \in E] \leq e^\varepsilon \mathbb{P}[\cA(\cD') \in E] + \delta.
\]
where the probability is taken over the randomness of $\cA$.
\end{definition}

Unlike GDP, $(\varepsilon, \delta)$-DP does not tightly handle composition. Yet, $(\varepsilon, \delta)$-DP offers more intuitive control over the privacy leakage (especially when $\delta$ is taken to be negligible) and is often referenced in legal, policy, and industry contexts as the standard formal privacy definition. It can thus be useful to convert GDP guarantees to $(\varepsilon,\delta)$-DP guarantees. 
\cite{desfontainesblog20240602} provides a helpful online tool for converting specific parameters.

\begin{lemma}[Conversion from $\gdpmu$-GDP to $(\varepsilon,\delta)$-DP \citep{Balle18, Dong22GDP}]
    \label{lem:convertion-approx-dp}
    Let $\cA$ be an algorithm satisfying $\gdpmu-$GDP. Then $\cA$ also satisfies $(\varepsilon, \delta)$-DP for any $\varepsilon > 0$ where 
    \[
        \delta = \Phi \left( - \frac{\varepsilon}{\gdpmu} + \frac{\gdpmu}{2} \right) - e^\varepsilon \Phi \left( - \frac{\varepsilon}{\gdpmu} - \frac{\gdpmu}{2} \right) \enspace ,
    \]
    where $\Phi$ is the CDF of the standard Gaussian distribution.
    %\me{who is $\Phi$?}
\end{lemma}

In our experiments, we report the privacy guarantees in both $\gdpmu$-GDP and the corresponding $(\varepsilon,\delta)$-DP guarantee obtained with Lemma~\ref{lem:convertion-approx-dp} for a fixed value of $\delta$. 

We note that it is also possible to convert GDP guarantees to other popular variants of DP, such as Rényi DP \citep[Corollary 3]{Mironov17} and zCDP \citep[Proposition 1.6]{BunSteinke16}. 

We use the standard composition property of GDP in our analysis.
\begin{lemma}[Composition of GDP~{\citep[Corollary~3.3]{Dong22GDP}}]
    \label{lem:composition}
    Let $\mathcal{A}_1$ and $\mathcal{A}_2$ denote a pair of algorithms that satisfies $\gdpmu_1$-GDP and $\gdpmu_2$-GDP, respectively. 
    Then the algorithm $\mathcal{A}(D) = (\mathcal{A}_1(D), \mathcal{A}_2(D))$ that outputs the result from both algorithms satisfies $\sqrt{\gdpmu^2_1 + \gdpmu_2^2}$-GDP.
\end{lemma}

\section{Details on private confidence intervals reporting}\label{app:details_CIs}

% \abin{it's a bit weird to have placed this appendix here (and not after the privacy-utility analysis, but let's keep it here for now to maintain consistent references}
\subsection{Bootstrapping}

Bootstrapping is a resampling-based method for estimating the variability of a statistic and constructing confidence intervals (CIs) without relying heavily on parametric assumptions. The principle is to repeatedly draw samples with replacement from the observed data (each of the same size as the original sample), compute the statistic of interest for each resample, and use the resulting empirical distribution to approximate its sampling distribution. Confidence intervals are then constructed from this bootstrap distribution, commonly using percentile-based or bias-corrected methods. For the obtained CIs to be valid, certain implicit assumptions must hold: the observed sample should be representative of the population, the data should be independent and identically distributed (i.i.d.), and the sample size should be sufficiently large for the resampling distribution to approximate the true sampling distribution. When these conditions are reasonably met, the bootstrap provides a flexible, data-driven way to quantify uncertainty \citep{efron1992bootstrap,efron1997improvements,davison1997bootstrap}.

We now detail how we incorporate bootstrapping in our framework, to provide differentially private confidence intervals.
We proceed as follows, detailing each step including the estimation step.
\begin{enumerate}
    \item \textbf{Data splitting.} Split the dataset into $K$ folds $(\cI_k)_{k\in[K]}$.
    \item \textbf{Nuisance estimation.} For $k\in[K]$, learn the nuisance estimation parameters $\hat\eta^{(k)}$ on $\cI_k$ (that could be propensity scores, outcome models, CATE), to obtain the score $\Gamma_i$ obtained via $\Gamma_i = \frac{1}{K-1}\sum_{k\ne k(i)}\phi(\hat\eta^{(k)},X_i,A_i,Y_i)$ exactly as in \Cref{sec:DP_estimators}.
    \item \textbf{Bootstrapping.} We now perform bootstrapping, to account for uncertainties in nuisance models and averaging.
    Let $R$ be the number of bootstraps to perform.
    For $k\in[K]$, let $\cI_k^{(b)}$ be a randomly sampled bootstrapped dataset from $\cI_k$, from which nuisance models $\hat\eta_b^{(k)}$ are learned. 
    This leads to scores $\Gamma_i^{(b)}=\frac{1}{K-1}\sum_{k\ne k(i)}\phi(\hat\eta_b^{(k)},X_i,A_i,Y_i)$ for $b\in[R]$.
    For some $\alpha\in(0,1/2)$, select $\Gamma_{-,i}$ and $\Gamma_{+,i}$ as respectively the $\alpha/2$ and $1-\alpha/2$ quantiles of the debiased  values $\set{\Gamma_i^{(b)}+\Gamma_i-\mathrm{median}(\Gamma_i^{(b)})_b\,,\, b\in[R]}$. 
    % \clin{Here we have $\Gamma_i = \frac{1}{R} \sum_{b \in [R]} \Gamma_i^{(b)}$? I'm a bit confused about the notation, do you also train new nuisance functions for each bootstrap, or do you reuse the same nuisance functions.}
    % \mein{Nope, $\Gamma_i$ is as before. And yes, new nuisance models for each bootstraps, that's the point, to take into account model uncertainty.}
    
    \item \textbf{Privatized CI.} Report the confidence interval $[\hat\tau_-,\hat\tau_+]$, where:
    \begin{equation*}
        \hat\tau_\pm \eqdef \frac{1}{n}\sum_{i=1}^n \Gamma_{\pm,i} + \cN(0,\sigma_1^2) \pm \Phi^{-1}(1-\beta/2)(\sigma_1 + \Bmu/(2\sqrt n))\,, 
    \end{equation*}
    % \clin{I don't think the notation $\sigma_1$ is needed here, there is no $\sigma_2$.}
    where $\Phi^{-1}$ is the inverse cdf of a Gaussian.
\end{enumerate}
Then, assuming that when one element between $\eta$, or $(X,A,Y)$ is changed, the value of $\phi(\eta,X,A,Y)$ changes by a maximum value of $4\Bmu$ (e.g. when $Y_i \in [-\Bmu,\Bmu]$ such that $(Y_i(1) - Y_i(0)) \in [-2\Bmu,2\Bmu]$), 
%$2\Bmu$, 
the privacy analysis of \Cref{thm:privacy_unified} can be applied for the constant of the G-Formula.

\subsection{Pointwise asymptotic normality of CATE estimates}

In causal forests \citep{athey2019estimating} for instance, the conditional average treatment effect (CATE) estimator enjoys pointwise asymptotic normality, meaning that for a fixed covariate value $x$, the estimator is approximately normally distributed around the true CATE as the sample size grows. This property allows us to construct valid confidence intervals for heterogeneous treatment effects, since the variance of the estimator can be consistently estimated using techniques such as the infinitesimal jackknife \citep{Efron03072014}. Thus, causal forests provide not only flexible, nonparametric estimates of treatment heterogeneity but also principled measures of uncertainty.

We here describe how we can obtain asymptotically valid and private CIs for the ATE, based on pointwise asymptotic normality and variance estimates.
We proceed as follows, detailing each steps as in the bootstrap method.

\begin{enumerate}
    \item \textbf{Data splitting.} Split the dataset into $K$ folds $(\cI_k)_{k\in[K]}$.
    \item \textbf{Nuisance estimation.} For $k\in[K]$, learn the nuisance estimation parameters $\hat\eta^{(k)}$ on $\cI_k$ (that could be propensity scores, outcome models, CATE), to obtain the score $\Gamma_i$ obtained via $\Gamma_i = \frac{1}{K-1}\sum_{k\ne k(i)}\phi(\hat\eta^{(k)},X_i,A_i,Y_i)$.
    For causal forests, $\hat \eta^{(k)}$ is an estimator of the CATE function $\mathrm{CATE}:x\mapsto\esp{Y_i|A_i=1,X_i=x}-\esp{Y_i|A_i=0,X_i=x}$ and $\phi(\hat\eta^{(k)},X_i,A_i,Y_i)=\hat\eta^{(k)}(X_i)$.
    \item \textbf{Pointwise variance.}
    The underlying assumption (true for instance for causal forests) is that the estimate $\phi(\hat\eta^{(k)},X_i,A_i,Y_i)$ satisfies $\phi(\hat\eta^{(k)},X_i,A_i,Y_i)\sim \cN(\mathrm{CATE}(X_i),V_k(X_i))$, asymptotically.
    The idea is thus to learn on $\cI_k$ an estimate $\hat V_k$ of the pointwise variance function $V_k$.
    Therefore, an asymptotically valid $1-\alpha$ coverage CI for $\mathrm{CATE}(X_i)$ is $[\phi(\hat\eta^{(k)},X_i,A_i,Y_i)\pm \Phi^{-1}(1-\alpha/2) \hat V_k(X_i)^{\frac{1}{2}}]$. However, this interval is \textit{pointwise}, while we require \textit{uniform} coverage for all of our $n$ samples and $K$ models.
    Therefore, using a Bonferroni correction, we get $[\phi(\hat\eta^{(k)},X_i,A_i,Y_i)\pm \Phi^{-1}(1-\alpha/(2nK)) \hat V_k(X_i)^{\frac{1}{2}}]$.
    % Then, since $\Gamma_i$ asymptotically satisfies:
    % \begin{equation*}
    %     \Gamma_i\sim \cN\left(\mathrm{CATE}(X_i), \frac{1}{(K-1)^2}\sum_{k\ne k(i)} V_k(X_i) \right)\,,
    % \end{equation*}
    % conditionally on the value of $X_i$, let $V_i$ the pointwise asymptotic variance of $\Gamma_i$ be defined as:
    % \begin{equation*}
    %     V_i\eqdef  \frac{1}{(K-1)^2}\sum_{k\ne k(i)} \hat V_k(X_i)\,.
    % \end{equation*}
    % Therefore, a valid $1-\alpha$ coverage CI for $\mathrm{CATE}(X_i)$ is $[\Gamma_i\pm \Phi^{-1}(1-\alpha/2) V_i^{\frac{1}{2}}]$. However, this interval is \textit{pointwise}, while we require \textit{uniform} coverage for all of our $n$ samples.
    % Therefore, using a Bonferroni correction, we get $[\Gamma_i\pm \Phi^{-1}(1-\alpha/(2n)) V_i^{\frac{1}{2}}]$ for our local CIs.
    Then, let 
    \begin{equation*}
        \Gamma_{i,\pm}\eqdef \frac{1}{K-1}\sum_{k\ne k(i)}\mathrm{proj}_{[-2\Bmu,2\Bmu]}\Big(\phi(\hat\eta^{(k)},X_i,A_i,Y_i)\pm \Phi^{-1}(1-\alpha/(2nK)) \hat V_k(X_i)^{\frac{1}{2}}\Big)\,,
        %\Gamma_{i,\pm}\eqdef \frac{1}{K-1}\sum_{k\ne k(i)}\mathrm{proj}_{[-\Bmu,\Bmu]}\Big(\phi(\hat\eta^{(k)},X_i,A_i,Y_i)\pm \Phi^{-1}(1-\alpha/(2nK)) \hat V_k(X_i)^{\frac{1}{2}}\Big)\,,
    \end{equation*}
    where $\mathrm{proj}_{[-2\Bmu,2\Bmu]}$ is the projection on the set $[-2\Bmu,2\Bmu]$. 
    %where $\mathrm{proj}_{[-\Bmu,\Bmu]}$ is the projection on the set $[-\Bmu,\Bmu]$.%\cl{Same as last section, project to -2B,2B? Is there a factor 2 off in the additive factor below?}

    \item \textbf{Privatized CI.} Report the confidence interval $[\hat\tau_-,\hat\tau_+]$, where:
    \begin{equation*}
        \hat\tau_\pm \eqdef \frac{1}{n}\sum_{i=1}^n \Gamma_{\pm,i} + \cN(0,\sigma_1^2) \pm \Phi^{-1}(1-\beta/2)(\sigma_1^2+\Bmu/(2\sqrt n))\,,
    \end{equation*}
    where $\Phi^{-1}$ is the inverse cdf of a Gaussian.
    %\clin{Where does $(\sigma_1^2+\Bmu/(2\sqrt n))$ come from? I'll have to check the proof more carefully to understand..}
    %\mein{the $\sigma_1^2$ comes from the Gaussian noise that we add (and that needs to be accounted for), while the other term comes from the empirical sum instead of population mean (the bootstrap in our case does not take that into account), but this will be super small.}
\end{enumerate}
Similarly to the bootstrap method, here the privacy analysis of \Cref{thm:privacy_unified} can be applied for the constant of the G-Formula.

\subsection{Proof of \Cref{thm:CIs}}\label{app:proof_theorem_CIs}

\begin{proof}
    The privacy part of the proof is a direct application of \Cref{prop:sensitivity_unified} with $\Delta_= = 4\Bmu$ and $\Delta_{\ne} = \frac{4K\Bmu}{K-1}$ (more details in the G-Formula privacy proof).

    Then,
    \begin{align*}
        \proba{\tau\in [\tau_\pm]} &=  \proba{\esp{\mathrm{CATE}(X_i)}\in \left[\frac{1}{n}\sum_{i\in[n]} \Gamma_{i,\pm}+ \cN(0,\sigma_1^2) \pm \Phi^{-1}(1-\beta/2)(\sigma_1 + \Bmu/(2\sqrt n))\right]}\\
        & \geq 1 - \proba{\left|\esp{\mathrm{CATE}(X_i)} - \cN(0,\sigma_1^2) - \frac{1}{n}\sum_{i=1}^n\mathrm{CATE}(X_i)\right| >  \Phi^{-1}(1-\beta/2)(\sigma_1+\Bmu/(2\sqrt n))} \\
        % &\quad - \proba{\cN(0,\sigma_1^2) \notin [\pm \Phi^{-1}(1-\beta/2)\sigma_1]}\\
        &\quad-   \proba{\frac{1}{n}\sum_{i=1}^n\mathrm{CATE}(X_i)\notin \left[\frac{1}{n}\sum_{i\in[n]} \Gamma_{i,\pm} \right]}\\
        &\geq 1-\beta -\proba{\frac{1}{n}\sum_{i=1}^n\mathrm{CATE}(X_i)\notin \left[\frac{1}{n}\sum_{i\in[n]} \Gamma_{i,\pm}\right]}\,,
    \end{align*}
    where the first probability is sub-Gaussiannity and the second inequality is a Gaussian bound.
    For the remaining probability, 
    \begin{align*}
        \proba{\frac{1}{n}\sum_{i=1}^n\mathrm{CATE}(X_i)\notin \left[\frac{1}{n}\sum_{i\in[n]} \Gamma_{i,\pm}\right]} & = \proba{\frac{1}{n}\sum_{i=1}^n\mathrm{CATE}(X_i)\in \left[\frac{1}{n(K-1)}\sum_{i\in[n],k\ne k(i)} \Gamma_{\pm}^{(k)}(i)\right]}\\
        & \leq \proba{\exists i\in[n],\exists k\ne k(i), \mathrm{CATE}(X_i) \notin \left[\Gamma_{\pm}^{(k)}(i)\right]}\\
        & \leq nK \alpha/(nK)\,,
    \end{align*}
    concluding the proof.
\end{proof}

\section{Unified analysis}
\label{app:unified_analysis_assumptotics}

\subsection{Proof of \Cref{thm:utility_general}}

\begin{proof}
	The proof follows from classical non-DP proofs that use crossfitting \citep[e.g., Theorem 3.2]{wager2024causal}, adapted to our setting. See \Cref{app:CI_var} below for details on CIs asymptotic validity.
\end{proof}

\subsection{Asymptotic Variance-based Confidence Intervals}\label{app:CI_var}

In the variance step of our method, we could generalize and instead privately estimate the variance of the estimator by adding an offset related to $\sigma_2^2$ in order to use the variance for confidence intervals as, for $\alpha\in(0,1)$ and $\alpha_1\in(0,p/2)$:
\begin{equation}
    \textstyle\hat V_\mathrm{DP}\eqdef \Big( \sqrt{\frac{1}{n-1}\sum_{i=1}^n\big(\hat\Gamma_i-\hat\tau\big)^2} + \cN(0,\sigma_2^2)\Big)^2+\sigma_1^2+\Phi^{-1}(1-\alpha_1/2)\sigma_2^2\,,
\end{equation}
where $\Phi(t)=\proba{\cN(0,1)\geq t}$ for $t\in\R$,  
and construct a private asymptotically valid $\alpha-$coverage confidence interval as:
\begin{equation}\label{eq:private_CI_p}
    \mathrm{CI}_\alpha \eqdef\big[ \hat \tau_\mathrm{DP} - \Phi^{-1}(1-\alpha/2+\alpha_1/2) \hat V_\mathrm{DP}^{1/2}/\sqrt{n}\,,\, \hat \tau_\mathrm{DP} + \Phi^{-1}(1-\alpha/2+\alpha_1/2) \hat V_\mathrm{DP}^{1/2}/\sqrt{n} \big]\,.
\end{equation}
Indeed, we first have that, for $\hat V =\frac{1}{n-1}\sum_{i=1}^n\big(\hat\Gamma_i-\hat\tau\big)^2$:
\begin{align*}
    &\proba{\Big( \sqrt{\frac{1}{n-1}\sum_{i=1}^n\big(\hat\Gamma_i-\hat\tau\big)^2} + \cN(0,\sigma_2^2)\Big)^2 + \Phi^{-1}(1-\alpha_1)\sigma_2^2 \geq \hat V  }\\
    & \geq  \proba{\cN(0,\sigma_2^2)^2\leq \Phi^{-1}(1-\alpha_1)^2\sigma_2^2}\\
    %& \geq  \proba{\cN(0,\sigma_2^2)^2\leq \Phi^{-1}(1-p_1)\sigma_2^2}\\
     & \geq  \proba{\cN(0,1)\leq \Phi^{-1}(1-\alpha_1)}\\
     &\geq 1-\alpha_1\,,
\end{align*}
by definition of $\Phi$.
Thus, with probability at least $1-\alpha_1$, we have that $\hat V_\mathrm{DP}\geq \hat V + \sigma_1^2$.
Then, \textit{in an oracle setting} where the estimated variance $\hat V$ is the actual variance $V^\star$ of the oracle estimator up to $o(1/n)$ and under asymptotic $\sqrt{n}-$consistency assumptions, under the central limit theorem, with $\hat \tau = \frac{1}{n}\sum_{i=1}^n\hat\Gamma_i$ and $c=\Phi^{-1}(1-\alpha/2+\alpha_1/2)$:
\begin{align*}
    \proba{\tau\in \mathrm{CI}_\alpha} &= \proba{\hat\tau + \cN(0,\sigma_1^2) \in\left[ \tau - c\sqrt{\hat V_\mathrm{DP}/n},\tau +c\sqrt{\hat V_\mathrm{DP}/n }\right]}\\
    &\geq \proba{\hat\tau + \cN(0,\sigma_1^2) \in\left[ \tau - c\sqrt{\hat V/n + \sigma_1^2/n},\tau + c\sqrt{\hat V/n + \sigma_1^2/n} \right]\Big|\hat V_\mathrm{DP} \geq \hat V+\sigma_1^2} - \alpha_1\\
    &\approx \proba{\cN(0,V^\star/n) + \cN(0,\sigma_1^2/n) \in\left[ \tau - c\sqrt{V^\star/n + \sigma_1^2/n},\tau + c\sqrt{V^\star/n + \sigma_1^2/n } \Big|\hat V_\mathrm{DP} \geq \hat V+\sigma_1^2 \right]} - \alpha_1\\
    &= \proba{\cN(0,\sqrt{V^\star/n+\sigma_1^2/n}) \in\left[ \tau - c\sqrt{V^\star /n+ \sigma_1^2/n},\tau + c\sqrt{V^\star/n + \sigma_1^2/n } \Big|\hat V_\mathrm{DP} \geq \hat V+\sigma_1^2 \right]} - \alpha_1\\
    &= 1-2\proba{\cN(0,1)\geq c } - \alpha_1\\
    & = 1- 2\alpha/2 +2\alpha_1/2 -\alpha_1\\
    &= 1-\alpha\,,
\end{align*}
where the first inequality uses the probabilistic bound above, and the $\approx$ holds up to a $o(1)$ under $\sqrt{n}-$consistency and exact oracle variance up to $o(1/n)$.
These assumptions are classically implied by those of \Cref{thm:utility_general} in the case of our estimators \citep{wager2024causal}.

\subsection{More general aggregation steps}

\begin{remark}\label{remark:nuisances}
The \textbf{aggregation step} can be made more general:
    \begin{equation}\label{eq:models_nuisances_general}
		\begin{aligned}
		    \hat\pi(i) = &\Phi_\pi\Big(\set{\hat\pi^{(k)}},k(i)\Big)\,,\quad \hat\mu_a(i) = \Phi_\mu\Big(\set{\hat\mu_a^{(k)}},k(i)\Big)\,,\\
		    &1-\tilde\pi(i) = \Phi_{1-\pi}\Big(\set{1-\hat\pi^{(k)}},k(i)\Big)\,,
		\end{aligned}
	    \end{equation}
Examples of functions $\Phi$ used in \Cref{eq:models_nuisances_general} are, for some models $(\hat m^{(k)})_{k\in[K]}$ (that could represent either $\hat\pi^{(k)},\hat\mu_a^{(k)}$ or 
$1-\hat\pi^{(k)}$):
\begin{enumerate}
    \item \textbf{Sampling.}
    For all $i\in[n]$, let
    \begin{equation}\label{eq:sampling}
        \Phi_\ell\Big(\set{\hat m^{(k)}},k(i)\Big) = \hat m^{(\ell(i))}\,,
    \end{equation}
	for $\ell(i)\in[K]\setminus\set{k(i)}$. For instance, $\ell(i)\sim\cU([K]\setminus\set{k(i)})$ can be chosen.
    \item \textbf{Mean.}
        For $\cK\subset[K]$, for all $i\in[n]$, let
    \begin{equation}\label{eq:mean}
        \Phi_\cK\Big(\set{\hat m^{(k)}},k(i)\Big) = \frac{1}{\#\cK-\one_\set{k(i)\in\cK}}\sum_{k\in\cK\setminus \set{k(i)}}\hat m^{(\ell)}\,.
    \end{equation}
    \item \textbf{Harmonic mean.}
    For $\cK\subset[K]$, if $m^{(k)}$ have positive values, for all $i\in[n]$, let
    \begin{equation}\label{eq:harmonic_mean}
        \Tilde\Phi_\cK\Big(\set{\hat m^{(k)}},k(i)\Big) = \left(\frac{1}{\#\cK-\one_\set{k(i)\in\cK}}\sum_{k\in\cK\setminus \set{k(i)}}\frac{1}{\hat m^{(\ell)}}\right)^{-1}\,.
    \end{equation}
\end{enumerate}
Other alternatives could be used, such as a median estimator (for its robustness properties), or more evolved ensemble methods to combine the models $m^{(k)}$.

In this paper and in all our results, propensity scores use the harmonic mean aggregator while conditional outcomes use the mean aggregator.
Motivations for this choice of aggregators stem from \Cref{lem:sensitivity_ensemble_models}.
Note that the sampling approach could be also used, for both propensity scores \textbf{and} conditional outcomes. 
We provide sensitivity analyses for this approach in the proofs in the appendices below.
\end{remark}

\subsection{Unified sensitivity analysis}

\begin{proposition}[Sensitivity analysis]\label{prop:sensitivity_unified}
Let 
\begin{equation*}
    \hat \tau = \frac{1}{n}\sum_{i=1}^n \hat \Gamma_i\,,
\end{equation*}
where $\hat \Gamma_i = \Psi\left((X_i,A_i,Y_i), k(i), \set{\hat\pi^{(k)}},\set{\hat\mu_a^{(k)}},\set{\tilde\pi^{(k)}}\right)$.
Assume that for all $i\in[n]$, for all datasets $\cD\sim_i\cD'$, we have that:
\begin{enumerate}
    \item $|\hat\Gamma_i-\hat\Gamma_i'|\leq \Delta_=$,
    \item For all $j\in\cI_{k(i)}\setminus\set{i}$, $\hat\Gamma_j=\hat\Gamma_j'$,
    \item  $\frac{1}{n}\sum_{j\notin \cI_{k(i)}}|\hat\Gamma_j-\hat\Gamma_j'|\leq \frac{\Delta_{\ne}}{K}$,
\end{enumerate}
where $(\hat\Gamma_j')$ corresponds to the quantities when models are trained and evaluated on $\cD'$.
Then, we have that:
\begin{equation*}
    \sup_{\cD\sim\cD'}|\hat\tau-\hat\tau'|\leq \frac{\Delta_=}{n}+\frac{\Delta_{\ne}}{K}\,,
\end{equation*}
so that for all $\sigma>0$, %$\alpha>1,\sigma>0$, 
the estimator $\hat\tau_\mathrm{DP}\eqdef\hat\tau + \cN(0,\sigma^2)$ is $\gdpmu-$GDP with $\gdpmu^2=\frac{\left(\frac{\Delta_=}{n}+\frac{\Delta_{\ne}}{K}\right)^2 }{\sigma^2}$.

Under the same assumptions and if furthemore $\sup_i|\hat\Gamma_i|\leq M$:
\begin{equation*}
    \sup_{\cD\sim\cD'}|\sqrt{\hat V}-\sqrt{\hat V'}|\leq \sqrt{\frac{2n}{(n-1)}}\left(2\sqrt M\left(\frac{\Delta_=}{n}+\frac{\Delta_{\ne}}{K}\right)^{1/2}+\frac{\Delta_=}{n}+\frac{\Delta_{\ne}}{K}\right)\,,
\end{equation*}
where $\hat V\eqdef \frac{1}{n-1}\sum_{i=1}^n\left(\hat\Gamma_i-\hat\tau\right)^2$.
Thus, for $\gamma_1,\gamma_2\sim\cN(0,1)$ independent and $\sigma_1,\sigma_2>0$, writing $\hat V_\mathrm{DP}= (\sqrt{\hat V} +\sigma_2\gamma_2)^2$ and $\hat\tau_\mathrm{DP}\eqdef\hat\tau + \cN(0,\sigma^2)$, we have that outputting the confidence interval
\begin{equation*}
    \cA(\cD)=\left[ \hat \tau_\mathrm{DP} - 1,96 \sqrt{\hat V_\mathrm{DP}} ,+, \hat \tau_\mathrm{DP} + 1,96 \sqrt{\hat V_\mathrm{DP}}\right]\,,
\end{equation*}
is $\gdpmu-$GDP, with $\gdpmu^2 = \frac{\left(\frac{\Delta_=}{n}+\frac{\Delta_{\ne}}{K}\right)^2 }{\sigma_1^2} + \frac{\left(\frac{\Delta_=}{n}+\frac{\Delta_{\ne}}{K} + (2M)^{\frac12} \left(\frac{\Delta_=}{n}+\frac{\Delta_{\ne}}{K}\right)^{\frac12}\right)^2}{\sigma_2^2}\frac{2n}{n-1}$.
\end{proposition}
\begin{proof}
We begin by proving that under the listed assumptions, we have:
\begin{equation*}
    \sup_{\cD\sim\cD'}|\hat\tau-\hat\tau'|\leq \frac{\Delta_=}{n}+\frac{\Delta_{\ne}}{K}\,.
\end{equation*}
Let $\cD\sim_i\cD'$.
We have:
\begin{align*}
    |\hat \tau-\hat\tau'| &= \left| \frac{1}{n}\sum_{j=1}^n \hat\Gamma_j-\hat\Gamma'_j\right|\\
    &\leq  \frac{1}{n}\sum_{j=1}^n \left|\hat\Gamma_j-\hat\Gamma'_j\right|\\
    &\leq  \frac{\left|\hat\Gamma_i-\hat\Gamma'_i\right|}{n} + \frac{1}{n}\sum_{j\in\cI_{k(i)}\setminus\set{i}}^n \left|\hat\Gamma_j-\hat\Gamma'_j\right|+\frac{1}{n}\sum_{j\notin\cI_{k(i)}}^n \left|\hat\Gamma_j-\hat\Gamma'_j\right|\\
    &\leq  \frac{\Delta_=}{n} + \frac{1}{n}\sum_{j\in\cI_{k(i)}\setminus\set{i}}^n 0+\frac{\Delta_{\ne}}{K}\\
    &\leq \frac{\Delta_=}{n}+\frac{\Delta_{\ne}}{K}\,,
\end{align*}
by a direct application of our assumptions.
The DP guarantees are then given by applying \Cref{lem:gaussian-mech-dp}.

We now show that
\begin{equation*}
    \sup_{\cD\sim\cD'}|\sqrt{\hat V}-\sqrt{\hat V'}|\leq \frac{n2}{n-1}\left(\frac{\Delta_=}{n}+\frac{\Delta_{\ne}}{K}\right)\,.
\end{equation*}
We have, for $\cD\sim_i\cD'$:
\begin{align*}
    \left|\sqrt{\hat V}-\sqrt{\hat V'}\right|&=\sqrt{\frac{1}{n-1}} \left|\sqrt{\sum_{i=1}^n\left(\hat\Gamma_i-\hat\tau\right)^2 }-\sqrt{\sum_{i=1}^n\left(\hat\Gamma_i'-\hat\tau'\right)^2 }\right|\\
    &\leq\sqrt{\frac{1}{n-1}} \sqrt{\sum_{i=1}^n\left(\hat\Gamma_i-\hat\tau -\set{\hat\Gamma_i'-\hat\tau'}\right)^2 }\\
    &\leq\sqrt{\frac{1}{n-1}} \sqrt{2\sum_{i=1}^n\left(\hat\Gamma_i -\hat\Gamma_i'\right)^2 + 2\sum_{i=1}^n\left(\hat\tau-\hat\tau'\right)^2 }\\
    &\leq\sqrt{\frac{2}{n-1}}\left( 2\sqrt M\sqrt{\sum_{i=1}^n\left|\hat\Gamma_i- \hat\Gamma_i'\right|}+\sqrt{n}|\hat\tau-\hat\tau'|\right) \\
    &\leq\sqrt{\frac{2}{n-1}}\left( 2\sqrt M\sqrt{n}\sqrt{\frac{1}{n}\sum_{i=1}^n\left|\hat\Gamma_i- \hat\Gamma_i'\right|}+\sqrt{n}|\hat\tau-\hat\tau'|\right) \\
    &\leq \sqrt{\frac{2n}{n-1}}\left(2\sqrt M\left(\frac{\Delta_=}{n}+\frac{\Delta_{\ne}}{K}\right)^{1/2}+\frac{\Delta_=}{n}+\frac{\Delta_{\ne}}{K}\right)\,,
\end{align*}
using the same arguments as above.
The DP guarantees are then given by applying \Cref{lem:gaussian-mech-dp}.

\end{proof}

\subsection{Sensitivity of our aggregators}

\begin{lemma}
    \label{lem:sensitivity_ensemble_models}
    For models $\set{m^{(k)}}$ trained on dataset $\cD$ as in \Cref{sec:framework} and $\hat m:[n]\to \R$ defined as $\hat m(i) = \Phi\left(\set{m^{(k)}},k(i)\right)$ as in \Cref{eq:models_nuisances}.
    Let $\cD\sim_i\cD'$ be two adjacent datasets, and denote as $\hat m'$ the global model corresponding to training on $\cD'$.
    We have the following properties for the examples detailed in \Cref{remark:nuisances}.
    \begin{enumerate}
        \item \textbf{Sampling.} Assume that $m^{(k)}$ lie in a space of diameter $M>0$.
        If $\Phi=\Phi_\ell$, for $\ell:[n]\to[K]$ such that $\ell(i)\ne k(i) $, then we have, for all $j\in [n]\setminus\cI_{k(i)}$:
        \begin{equation*}
            |\hat m(j)-\hat m'(j)| \leq M\one_\set{j\in\cJ_{k(i)}}\,,
        \end{equation*}
        where $\cJ_k=\set{i:\ell(i)=k}$,
        and $\hat m(j)=\hat m'(j)$ for $j\in \cI_{k(i)}$.
        
        \item \textbf{Mean.} Assume that $m^{(k)}$ lie in a space of diameter $M>0$.
        If $\Phi=\Phi_\cK$ for $\cK\subset[K]$, then we have, for all $j\in [n]\setminus\cI_{k(i)}$:
        \begin{equation*}
            |\hat m(j)-\hat m'(j)| \leq \frac{M}{\#\cK -\one_\set{k(i)\in\cK}}\,,
        \end{equation*}
        and $\hat m(j)=\hat m'(j)$ for $j\in \cI_{k(i)}$.
        
        \item \textbf{Harmonic mean (1).} Assume that $m^{(k)}>0$ and that $1/m^{(k)}$ are uniformly bounded by a constant $M>0$.
        If $\Phi=\Tilde\Phi_\cK$ for $\cK\subset[K]$, then we have, for all $j\in [n]\setminus\cI_{k(i)}$:
        \begin{equation*}
            \left|\frac{1}{\hat m(j)}-\frac{1}{\hat m'(j)}\right| \leq \frac{M}{\#\cK -\one_\set{k(i)\in\cK}}\,,
        \end{equation*}
        and $\hat m(j)=\hat m'(j)$ for $j\in \cI_{k(i)}$.

        \item \textbf{Harmonic mean (2).} Assume that $0<m^{(k)}$ and $1/m^{(k)}$ are uniformly bounded by a constant $M>0$.
        If $\Phi=\Tilde\Phi_\cK$ for $\cK\subset[K]$, then we have, for all $j\in [n]\setminus\cI_{k(i)}$:
        \begin{equation*}
            |\hat m(j)-\hat m'(j)| \leq \frac{M^3}{\#\cK -\one_\set{k(i)\in\cK}}\,,
        \end{equation*}
        and $\hat m(j)=\hat m'(j)$ for $j\in \cI_{k(i)}$.
    \end{enumerate}
\end{lemma}

The harmonic mean in its second form above is thus way less competitive and thus should only be used for propensity scores.

\begin{proof}
The proof of the first, second and third points are straightforward.
For the harmonic mean, let $j\in[n]\setminus\cI_{k(i)}$.
We have that
\begin{equation*}
    \hat m(j)=\frac{1}{\frac{1}{K}\sum_{k\in\cK\setminus\set{k(i)}} \hat m^{(k)}(X_j)^{-1} +  \frac{\hat m^{(k(i))}(X_j)^{-1}}{K}}\,,
\end{equation*}
and
\begin{equation*}
    \hat m'(j)=\frac{1}{\frac{1}{K}\sum_{k\in\cK\setminus\set{k(i)}} \hat m^{(k)}(X_j)^{-1} +  \frac{\hat m^{'(k(i))}(X_j)^{-1}}{K}}\,.
\end{equation*}
Let $C=\frac{1}{K}\sum_{k\in\cK\setminus\set{k(i)}} \hat m^{(k)}(X_j)^{-1}\geq \frac{1}{M}$, and
note that $f:x>0\mapsto \frac{1}{C+\frac{x}{K}}$ is convex.
Thus, for $0<x_1<x_2\leq M$,
\begin{align*}
    0\leq f(x_1)-f(x_2)&\leq f'(x_1)(x_1-x_2)\\
    &\leq |f'(0)|(x_2-x_1)\\
    &=\frac{x_2-x_1}{KC^2}\\
    &\leq \frac{M^2(x_2-x_1)}{K}\\
    &\leq  \frac{M^3}{K}\,.
\end{align*}
\end{proof}

\section{Analysis of the DP-G-Formula estimator}

\begin{lemma}[G-Formula]
    \label{lem:g-formula}
    For $\ell:[n]\to[K]$ such that for all $i\in[n]$ we have $\ell(i)\in[K]\setminus\set{k(i)}$, let
    \begin{equation*}
        \Phi_{\mu} \in\set{\Phi_{[K]}, \Phi_\ell}\,,
    \end{equation*}
    as defined in \Cref{eq:sampling}. 
    Let $\cJ_k\eqdef \set{i\in[n]:\ell(i)=k}$ for $k\in[K]$.
    Then, for the DP-G-Formula estimator, the assumptions of \Cref{prop:sensitivity_unified} are verified with:
    \begin{equation*}
        \Delta_= = 4\Bmu \,,
    \end{equation*}
    and
    \begin{equation*}
        \Delta_{\ne} = 4\Bmu\times\frac{K}{K-1}\,.
    \end{equation*}
    if $\Phi_{\mu} =\Phi_{[K]}$, or
    \begin{equation*}
        \Delta_{\ne} = 4\Bmu\times \frac{K \sup_{k\in[K]}\# \cJ_k}{n}\,,
    \end{equation*}
    if $\Phi_{\mu} =\Phi_{\ell}$.
\end{lemma}

\begin{proof}
	Let $\cD\sim\cD'$ and assume without loss of generality that $(X_i,A_i,Y_i)\ne(X_i,A_i,Y_i)$ for $i=1$ and that $k(1)=1$ (i.e., $1\in\cI_1$).
	Denote by $\hat\mu_t^{',(k)},\hat\mu_t',\hat\pi',\hat\pi^{',(k)}$ the models learned on $\cD'$.
	Since $\cD$ and $\cD'$ differ only on the datapoint $i=1$ and that $1\notin\cI_k$ for $k\geq 2$, we have that $\hat\mu_a^{',(k)}=\hat\mu_a^{(k)}$ for $k\geq 2$.
	
    First, for all $j\in\cI_{1}\setminus\set{1}$, we have $(X_i,A_i,T_i)=(X'_i,A_i',T_i')$ and $\hat\mu_a(i)=\hat\mu'_a(i)$, so that $\hat \Gamma_i=\hat\Gamma'_i$.
    
    Then, for $i=1$:
    \begin{align*}
        &|\hat\Gamma_1-\hat\Gamma_1'|\leq |\hat\mu_1(1)-\hat \mu_0(1) -\hat\mu_1'(1)+\hat \mu_0'(1)| \\
        &\leq |\hat\mu_1(1)|+|\hat \mu_0(1)|+|\hat\mu_1'(1)|+|\hat \mu_0'(1)|\\
        &\leq 4B\,,
    \end{align*}
	where we used \Cref{asm:boundedness}.
    This expression is independent of $i=1$, so we can conclude that 
    \begin{equation*}
        \Delta_= = 4\Bmu\,.
    \end{equation*}
    
    We then find a bound for $\Delta_{\ne}$.
    
    \textbf{Case 1:} $\Phi_\mu=\Phi_\ell$.
    Let $\cJ_k=\set{i\in[n]\text{ s.t. }\ell(i)=k}$.
    We have $\cJ_1\subset[n]\setminus \cI_1$.
    For all $j\notin\cJ_1$, we have $\hat\mu_a(j)=\hat\mu'_a(j)$.
    For $i\in\cJ_1$, using \Cref{lem:sensitivity_ensemble_models} and \Cref{asm:boundedness},
	\begin{align*}
    	|\hat\Gamma_i-\hat\Gamma_i'|\leq 4\Bmu\,.
	\end{align*}
    Thus, by summing over all $j\notin\cI_1$,
    \begin{equation*}
        \Delta_{\ne} = 4\Bmu\times \frac{K \sup_{k\in[K]}\# \cJ_k}{n}\,.
    \end{equation*}
    
    \textbf{Case 2:} $\Phi_\mu=\Phi_{[K]}$.
    Using \Cref{lem:sensitivity_ensemble_models} and \Cref{asm:boundedness}, for all $i\notin \cI_1$, we have:
    \begin{equation*}
        |\hat\Gamma_i-\hat\Gamma_i'|\leq \frac{4\Bmu}{K-1}\,.
    \end{equation*}
    By summing, we thus obtain:
    \begin{equation*}
        \Delta_{\ne}=\frac{4\Bmu K}{K-1}\,.
    \end{equation*}
\end{proof}

\section{Analysis of the DP-IPW estimator}

\begin{lemma}[IPW]
    \label{lem:ipw}
    For $\ell:[n]\to[K]$ such that for all $i\in[n]$ we have $\ell(i)\in[K]\setminus\set{k(i)}$, let
    \begin{equation*}
        \Phi_{\pi},\, \Phi_{1-\pi} \in\set{\Tilde \Phi_{[K]}, \Phi_\ell}\,,
    \end{equation*}
    as defined in \Cref{eq:sampling}. 
    Let $\cJ_k\eqdef \set{i\in[n]:\ell(i)=k}$ for $k\in[K]$.
    Then, for the DP-IPW estimator, the assumptions of \Cref{prop:sensitivity_unified} are verified with:
    \begin{equation*}
        \Delta_= = 2\Bmu\Bpi \,,
    \end{equation*}
    and
    \begin{equation*}
        \Delta_{\ne} = \Bmu\Bpi\times\frac{K}{K-1}\,.
    \end{equation*}
    if $\Phi_{\pi},\, \Phi_{1-\pi} =\Tilde \Phi_{[K]}$, or
    \begin{equation*}
        \Delta_{\ne} =\Bmu\Bpi\times \frac{K \sup_{k\in[K]}\# \cJ_k}{n}\,,
    \end{equation*}
    if $\Phi_{\pi},\, \Phi_{1-\pi} =\Phi_{\ell}$.
\end{lemma}

\begin{proof}
	Let $\cD\sim\cD'$ and assume without loss of generality that $(X_i,A_i,Y_i)\ne(X_i,A_i,Y_i)$ for $i=1$ and that $k(1)=1$ (i.e., $1\in\cI_1$).
	Denote by $\hat\mu_t^{',(k)},\hat\mu_t',\hat\pi',\hat\pi^{',(k)}$ the models learned on $\cD'$.
	Since $\cD$ and $\cD'$ differ only on the datapoint $i=1$ and that $1\notin\cI_k$ for $k\geq 2$, we have that $\hat\pi^{',(k)}=\hat\pi^{(k)}$ for $k\geq 2$.
	
    First, for all $j\in\cI_{1}\setminus\set{1}$, we have $(X_i,A_i,T_i)=(X'_i,A_i',T_i')$ and $\hat\pi(i)=\hat\pi'(i)$ and $\tilde\pi(i)=\tilde\pi'(i)$, so that $\hat \Gamma_i=\hat\Gamma'_i$.
    
    Then, for $i=1$, since $\hat\pi(1)=\hat\pi'(1)$ and $\tilde\pi(1)=\tilde\pi'(1)$:
    \begin{align*}
        |\hat\Gamma_1-\hat\Gamma_1'|&= |\frac{A_1 Y_1}{\hat\pi(1)}-\frac{A_1' Y_1'}{\hat\pi'(1)} - \frac{(1-A_1)Y_1}{1-\tilde\pi(1)}+ \frac{(1-A_1')Y_1'}{1-\tilde\pi'(1)}| \\
        &\leq \Bmu (A_1+1-A_1+A_1'+1-A_1')\Bpi\,,
    \end{align*}
	where we used \Cref{asm:boundedness}.
    This expression is independent of $i=1$, so we can conclude that 
    \begin{equation*}
        \Delta_= = 2\Bmu\Bpi\,.
    \end{equation*}
    
    We then find a bound for $\Delta_{\ne}$.
    
    \textbf{Case 1:} $\Phi_\mu=\Phi_\ell$.
    Let $\cJ_k=\set{i\in[n]\text{ s.t. }\ell(i)=k}$.
    We have $\cJ_1\subset[n]\setminus \cI_1$.
    For all $j\notin\cJ_1$, we have $\hat\pi(j)=\hat\pi'(j)$ and $\tilde\pi(j)=\tilde\pi'(j)$.
    Thus, for $i\in\cJ_1$, using \Cref{lem:sensitivity_ensemble_models} and \Cref{asm:boundedness}, since $A_i=A_i'$ and $Y_i=Y_i'$:
	\begin{align*}
    	|\hat\Gamma_i-\hat\Gamma_i'| = |A_iY_i(1/\hat\pi(i)-1/\hat\pi'(i))+(1-A_i)Y_i(1/(1-\hat\pi(i))-1/(1-\hat\pi'(i))| \leq \Bmu\Bpi\,.
	\end{align*}
    Thus, by summing over all $j\notin\cI_1$,
    \begin{equation*}
        \Delta_{\ne} = \Bmu\Bpi\times \frac{K \sup_{k\in[K]}\# \cJ_k}{n}\,.
    \end{equation*}
    
    \textbf{Case 2:} $\Phi_\mu\in\set{\Phi_{[K]},\Tilde \Phi_{[K]}}$.
    Using \Cref{lem:sensitivity_ensemble_models} and \Cref{asm:boundedness}, for all $i\notin \cI_1$, we have:
    \begin{equation*}
        |\hat\Gamma_i-\hat\Gamma_i'|\leq \Bmu\Bpi\times\frac{1}{K-1}\,.
    \end{equation*}
    By summing, we thus obtain:
    \begin{equation*}
        \Delta_{\ne}=\Bmu\Bpi\times \frac{K}{K-1}\,.
    \end{equation*}
\end{proof}

\section{Analysis of the DP-AIPW estimator}

\begin{lemma}[AIPW - Sampling]
    \label{lem:aipw_sampling}
    For $\ell:[n]\to[K]$ such that for all $i\in[n]$ we have $\ell(i)\in[K]\setminus\set{k(i)}$, let
    \begin{equation*}
        \Phi_\pi = \Phi_{1-\pi} = \Phi_{\mu} = \Phi_\ell\,,
    \end{equation*}
    as defined in \Cref{eq:sampling}. 
    Let $\cJ_k\eqdef \set{i\in[n]:\ell(i)=k}$ for $k\in[K]$.
    Then, for the DP-AIPW estimator, the assumptions of \Cref{prop:sensitivity_unified} are verified with:
    \begin{equation*}
        \Delta_= = 2\Bmu(2+\Bpi) \,,
    \end{equation*}
    and
    \begin{equation*}
        \Delta_{\ne} = \left(4\Bmu + 3\Bmu\Bpi\right)\times \frac{K \sup_{k\in[K]}\# \cJ_k}{n}\,.
    \end{equation*}
\end{lemma}
\begin{proof}
	Let $\cD\sim\cD'$ and assume without loss of generality that $(X_i,A_i,Y_i)\ne(X_i,A_i,Y_i)$ for $i=1$ and that $k(1)=1$ (i.e., $1\in\cI_1$).
	Denote by $\hat\mu_t^{',(k)},\hat\mu_t',\hat\pi',\hat\pi^{',(k)}$ the models learned on $\cD'$.
	Since $\cD$ and $\cD'$ differ only on the datapoint $i=1$ and that $1\notin\cI_k$ for $k\geq 2$, we have that $\hat\mu_a^{',(k)}=\hat\mu_a^{(k)}$ and $\hat\pi^{',(k)}=\hat\pi^{(k)}$ for $k\geq 2$.
	
    First, for all $j\in\cI_{1}\setminus\set{1}$, we have $(X_i,A_i,T_i)=(X'_i,A_i',T_i')$ and $\hat\mu_a(i)=\hat\mu'_a(i)$, so that $\hat \Gamma_i=\hat\Gamma'_i$.
    
    Then, for $i=1$:
    \begin{align*}
        &|\hat\Gamma_1-\hat\Gamma_1'|\leq |\hat\mu_1(1)-\hat \mu_0(1) -\hat\mu_1'(1)+\hat \mu_0'(1)| \\
        &+ \left|\frac{A_i}{\hat\pi(i)}(Y_i-\hat\mu_1(i)) + \frac{1-A_i}{1-\hat\pi(i)}(Y_i-\hat\mu_0(i))-\set{ \frac{A_i}{\hat\pi'(i)}(Y_i-\hat\mu_1'(i)) + \frac{1-A_i}{1-\hat\pi'(i)}(Y_i-\hat\mu_0'(i))}\right|\\
        &\leq |\hat\mu_1(1)|+|\hat \mu_0(1)|+|\hat\mu_1'(1)|+|\hat \mu_0'(1)|\\
        &+ \left|\frac{A_1}{\hat\pi(1)}(Y_1-\hat\mu_1(1))\right| + \left|\frac{1-A_1}{1-\hat\pi(1)}(Y_1-\hat\mu_0(1))\right| +\left|\frac{A'_1}{\hat\pi'(1)}(Y'_1-\hat\mu_1'(1))\right| +\left| \frac{1-A'_1}{1-\hat\pi'(1)}(Y'_1-\hat\mu_0'(1)) \right|\\
        &\leq 4\Bmu + 2A_1\Bmu\Bpi+2(1-A_1)\Bmu\Bpi+2A_1'\Bmu\Bpi+2(1-A_1')\Bmu\Bpi\\
        &\leq 2\Bmu(2+\Bpi)\,,
    \end{align*}
	where we used \Cref{asm:boundedness} to bound $1/\hat\pi'(1),1/\hat\pi(1),1/(1-\hat\pi'(1)),1/(1-\hat\pi(1))$.
    This expression is independent of $i=1$, so we can conclude that 
    \begin{equation*}
        \Delta_= = 2\Bmu(2+\Bpi)\,.
    \end{equation*}
    We then find a bound for $\Delta_{\ne}$.
    Let $\cJ_k=\set{i\in[n]\text{ s.t. }\ell(i)=k}$.
    We have $\cJ_1\subset[n]\setminus \cI_1$.
    For all $j\notin\cJ_1$, we have $\hat\pi(j)=\hat\pi'(j)$, $\tilde\pi(j)=\tilde\pi'(j)$, and $\hat\mu_a(j)=\hat\mu'_a(j)$.
    For $i\in\cJ_1$, 
	\begin{align*}
    	&|\hat\Gamma_i-\hat\Gamma_i'|\\
		&=\left|\frac{A_i}{\hat\pi(i)}(Y_i-\hat\mu_1(i)) + \frac{1-A_i}{1-\hat\pi(i)}(Y_i-\hat\mu_0(i))-\set{ \frac{A_i}{\hat\pi'(i)}(Y_i-\hat\mu_1(i)) + \frac{1-A_i}{1-\hat\pi'(i)}(Y_i-\hat\mu_0(i))}\right|\\
		&\leq  A_i\abs{\frac{Y_i-\hat\mu_1(i)}{\hat\pi(i)}-\frac{Y_i-\hat\mu_1(i)}{\hat\pi'(i)}} + (1-A_i)\abs{\frac{Y_i-\hat\mu_0(i)}{1-\hat\pi(i)}-\frac{Y_i-\hat\mu_0(i)}{1-\hat\pi'(i)}} \\
		&\leq 2\Bmu A_i\abs{\frac{1}{\hat\pi(i)}-\frac{1}{\hat\pi'(i)}} + 2\Bmu (1-A_i)\abs{\frac{\hat\mu_0(i)}{\hat\pi(i)}-\frac{\hat\mu_0'(i)}{\hat\pi'(i)}} \\
		&= 2\Bmu\Bpi\,.
	\end{align*}
	Then, for $j\in\cJ_1'$, we have $\hat\pi(i)=\hat\pi'(i)$, and thus:
		\begin{align*}
			&|\hat\Gamma_i-\hat\Gamma_i'|\\
    		&=\left|\hat\mu_1(i)-\hat \mu_0(i) + \frac{A_i}{\hat\pi(i)}(Y_i-\hat\mu_1(i)) + \frac{1-A_i}{1-\hat\pi(i)}(Y_i-\hat\mu_0(i))\right.\\
    		&\quad\left.-\set{\hat\mu_1'(i)-\hat \mu_0'(i) + \frac{A_i}{\hat\pi'(i)}(Y_i-\hat\mu_1'(i)) + \frac{1-A_i}{1-\hat\pi'(i)}(Y_i-\hat\mu_0'(i))}\right|\\
		&\leq 4\Bmu +A_i\abs{Y_i}\abs{\frac{1}{\hat\pi(i)}-\frac{1}{\hat\pi'(i)}} + A_i\abs{\frac{\hat\mu_1(i)}{\hat\pi(i)}-\frac{\hat\mu_1'(i)}{\hat\pi'(i)}} + (1-A_i)\abs{Y_i}\abs{\frac{1}{\hat\pi(i)}-\frac{1}{\hat\pi'(i)}} + (1-A_i)\abs{\frac{\hat\mu_0(i)}{\hat\pi(i)}-\frac{\hat\mu_0'(i)}{\hat\pi'(i)}} \\
		&\leq 4\Bmu + (\Bmu\Bpi + 2\Bmu\Bpi)(A_i+1-A_i)\\
		&= 4\Bmu + 3\Bmu\Bpi\,.
	\end{align*}
    Thus, by summing over all $j\notin\cI_1$,
    \begin{equation*}
        \Delta_{\ne} = \left(4\Bmu + 3\Bmu\Bpi\right)\times \frac{K \sup_{k\in[K]}\# \cJ_k}{n}\,.
    \end{equation*}
\end{proof}

\begin{lemma}[AIPW - Complete Means]
    \label{lem:aipw_means}
    Let 
    \begin{equation*}
        \Phi_\pi , \Phi_{1-\pi} , \Phi_{\mu} =(\Tilde\Phi_{[K]},\Tilde \Phi_{[K]},\Phi_{[K]})\,,
    \end{equation*}
    as defined in \Cref{eq:mean,eq:harmonic_mean}: the ensembling is done via mean or harmonic mean over all samples (except those in the fold over which the sample models are evaluated on is in). 
    Then, for the DP-AIPW estimator, the assumptions of \Cref{prop:sensitivity_unified} are verified with:
    \begin{equation*}
        \Delta_= = 4\Bmu(1+\Bpi) \,,
    \end{equation*}
    and
    \begin{equation*}
        \Delta_{\ne} = \frac{4\Bmu K}{K-1}\left(1+\Bpi\right)\,.
    \end{equation*}
\end{lemma}
\begin{proof}
	Let $\cD\sim\cD'$ and assume without loss of generality that $(X_i,A_i,Y_i)\ne(X_i,A_i,Y_i)$ for $i=1$ and that $k(1)=1$ (i.e., $1\in\cI_1$).
	Denote by $\hat\mu_t^{',(k)},\hat\mu_t',\hat\pi',\hat\pi^{',(k)}$ the models learned on $\cD'$.
	Since $\cD$ and $\cD'$ differ only on the datapoint $i=1$ and that $1\notin\cI_k$ for $k\geq 2$, we have that $\hat\mu_a^{',(k)}=\hat\mu_a^{(k)}$ and $\hat\pi^{',(k)}=\hat\pi^{(k)}$ for $k\geq 2$.

    First, for all $j\in\cI_{1}\setminus\set{1}$, we have $(X_i,A_i,T_i)=(X'_i,A_i',T_i')$ and $\hat\mu_a(i)=\hat\mu'_a(i)$ (they are computed using only models trained on folds $k\geq 2$ that are thus unchanged), so that $\hat \Gamma_i=\hat\Gamma'_i$.

    Then, for $i=1$, we also have that the models are unchanged (trained on folds $k\geq 2$), so that:
    \begin{align*}
        &|\hat\Gamma_1-\hat\Gamma_1'|\leq |\hat\mu_1(1)-\hat \mu_0(1) -\hat\mu_1'(1)+\hat \mu_0'(1)| \\
        &+ \left|\frac{A_i}{\hat\pi(i)}(Y_i-\hat\mu_1(i)) + \frac{1-A_i}{1-\hat\pi(i)}(Y_i-\hat\mu_0(i))-\set{ \frac{A_i}{\hat\pi'(i)}(Y_i-\hat\mu_1'(i)) + \frac{1-A_i}{1-\hat\pi'(i)}(Y_i-\hat\mu_0'(i))}\right|\\
        &\leq |\hat\mu_1(1)|+|\hat \mu_0(1)|+|\hat\mu_1'(1)|+|\hat \mu_0'(1)|\\
        &+ \left|\frac{A_1}{\hat\pi(1)}(Y_1-\hat\mu_1(1))\right| + \left|\frac{1-A_1}{1-\hat\pi(1)}(Y_1-\hat\mu_0(1))\right| +\left|\frac{A'_1}{\hat\pi'(1)}(Y'_1-\hat\mu_1'(1))\right| +\left| \frac{1-A'_1}{1-\hat\pi'(1)}(Y'_1-\hat\mu_0'(1)) \right|\\
        &\leq 4\Bmu + 2A_1\Bmu\Bpi+2(1-A_1)\Bmu\Bpi+2A_1'\Bmu\Bpi+2(1-A_1')\Bmu\Bpi\\
        &\leq 4\Bmu(1+\Bpi)\,,
    \end{align*}
	where we used \Cref{asm:boundedness} to bound $1/\hat\pi'(1),1/\hat\pi(1),1/(1-\hat\pi'(1)),1/(1-\hat\pi(1))$/
    This expression is independent of $i=1$, so we can conclude that 
    \begin{equation*}
        \Delta_= = 4\Bmu(1+\Bpi)\,.
    \end{equation*}
    
    Now using \Cref{lem:sensitivity_ensemble_models} and \Cref{asm:boundedness}, we have that for $i\in[n]\setminus \cI_1$,
    \begin{align*}
        |\hat \pi(i)^{-1}-\hat\pi'(i)^{-1}|\leq & \Bpi/(K-1)\,,\quad|(1-\tilde \pi(i))^{-1}-(1-\tilde\pi'(i))^{-1}|\leq \Bpi/(K-1)\,,\\
        &\quad |\hat \mu_a(i)-\hat\mu_a'(i)|\leq 2\Bmu/(K-1)\,.
    \end{align*}
    	Thus, for $i\notin\cI_1$,
	\begin{align*}
		|\hat\mu_1(i)-\hat \mu_0(i)-\set{\hat\mu_1'(i)-\hat \mu_0'(i)}|&\leq |\hat\mu_1(i)-\hat\mu_1'(i)|+|\hat\mu_0(i)-\hat\mu_0'(i)|\\
		&\leq \frac{4\Bmu}{K-1}\,,
	\end{align*}
	and 
	\begin{align*}
		\left|\frac{A_i}{\hat\pi(i)}(Y_i-\hat\mu_1(i))-\frac{A_i}{\hat\pi'(i)}(Y_i-\hat\mu_1'(i))\right|&\leq A_i\left|\frac{\hat\mu(i)-\hat\mu'(i)}{\hat\pi(i)}\right| + A_i\left|(Y_i-\hat\mu'(i))\left(\frac{1}{\hat\pi(i)}-\frac{1}{\hat\pi'(i)}\right)\right|\\
		& \leq \frac{4\Bmu\Bpi A_i}{K-1}\,.
	\end{align*}
		Similarly,
	\begin{align*}
		\left|\frac{1-A_i}{1-\hat\pi(i)}(Y_i-\hat\mu_0(i))-\frac{1-A_i}{1-\hat\pi'(i)}(Y_i-\hat\mu_0'(i))\right| \leq \frac{4\Bmu\Bpi(1-A_i)}{K-1}\,.
	\end{align*}
    Thus,
    \begin{equation*}
        |\hat \Gamma_i-\hat\Gamma_i'|\leq \frac{4\Bmu}{K-1}\left(1+\Bpi\right).
    \end{equation*}
    We can thus choose:
    \begin{equation*}
        \Delta_{\ne}=\frac{4\Bmu K}{K-1}\left(1+\Bpi\right)\,.
    \end{equation*}

\end{proof}

\section{Meta-analysis of differentially private ATEs}
\label{app:meta-analyses}

Meta-analyses are statistical methods for combining the results from $N\geq2$ independent studies to increase statistical power \citep{hunter2004methods,borenstein2021introduction}. They are considered the pinnacle of evidence in clinical research hierarchies. In this section, we show how to combine several \emph{private} ATE estimates via meta-analysis to achieve a lower-variance estimate.

% In our setting of treatment effect estimation quantified in terms of ATE with the RD, this consists in combining ATEs estimated on different studies, usually with a weighted average of the corresponding results.
% More precisely, let $N$ be the number of studies or centers considered.
We assume each study $j\in[N]$ releases an ATE estimate $\hat\tau_\mathrm{DP}^{(j)}$ and its estimated variance $\hat V_\mathrm{DP}^{(j)}$, computed from an independent sample of size $n_j$ drawn from the same population, with $\gdpmu^{(j)}$-GDP guarantees.
% For $j\in[N]$, we assume that study/center $j$ has access to a dataset with $n_j$ samples.
% In a private meta-analysis, the owner of each study may impose that their data remains private, via differential privacy.
Constructing such estimates can be handled using the framework presented in Section~\ref{sec:DP_estimators}.
Note that each release can be conducted independently---in particular, we do not require the studies to use the same ATE estimator.
% In this model, our approach goes beyond and is more flexible than that of \citet{KogaCP24}, that relied on RCT assumptions or on perfect matchings in observational studies. 
% The framework presented in Section~\ref{sec:DP_estimators} can be used for this.
% We further assume that each estimate $\hat\tau_\mathrm{DP}^{(j)}$ is asymptotically unbiased w has (estimated) local variance $\hat V_\mathrm{DP}^{(j)}$. While other approaches could be used, in the following we will assume that the framework presented in Section~\ref{sec:DP_estimators}
% In that case, our methodology can be used: each center $j\in[N]$ computes $\hat\tau_\mathrm{DP}^{(j)}$ a DP estimate of $\esp{Y_i^{(j)}(1)-Y_i^{(j)}(0)}$ with an estimated local variance $\hat V_\mathrm{DP}^{(j)}$, \textit{with any of our methodologies}.

Given some weights $\lambda_1,\dots,\lambda_N\geq 0$ with $\sum_{j=1}^N\lambda_j=1$, the meta-analysis ATE estimate then writes as
$\hat\tau_\mathrm{DP-meta}=\sum_{j=1}^N \lambda_j\hat\tau_\mathrm{DP}^{(j)}$, where $\lambda_1,\dots,\lambda_N\geq 0$, which is asymptotically unbiased with an estimated variance $\hat V_\mathrm{DP-meta}/n=\sum_{j=1}^N \lambda_j^2\hat V_\mathrm{DP}^{(j)}/n_j$, where $n = \sum_{j \in [N]} n_j$ is the combined sample size. 
Importantly, as long as the weights are determined independently of private data, releasing $\hat\tau_\mathrm{DP-meta}$ does not affect the differential privacy guarantees, since it is a post-processing of already privatized estimates.
% We assume that study/center $j$ imposes $\gdpmu_1^{(j)}-$GDP for $\hat\tau_\mathrm{DP}^{(j)}$ and $\sqrt{\gdpmu_1^{2(j)}+\gdpmu_2^{2(j)}}-$GDP for outputting both the private ATE estimate and its privately estimated variance.
% These DP estimates inherit the properties of our analysis, they are asymptotically unbiased and have a known asymptotic variance given by \Cref{thm:privacy_unified,thm:utility_general}.

% Our privatized estimators and variance offer weighting strategies to minimize the variance of the output of the meta analysis. 

% Note that each release can be conducted independently, as we do not require the study/centers to use the same ATE estimator.
% In this model, our approach goes beyond and is more flexible than that of \citet{KogaCP24}, that relied on RCT assumptions or on perfect matchings in observational studies. 
% Our privatized estimators and variance offer weighting strategies to minimize the variance of the output of the meta analysis. 
% If all $\hat\tau^{(j)}_\mathrm{DP}$ are privately released, then outputting the aggregated estimator $\hat\tau_\mathrm{DP}$ does not change privacy guarantees (see \cite[Proposition~2.8]{Dong22GDP}): the objective for the server is thus to minimize the variance of this aggregated meta-analysis ATE estimator, which can be done as follows.

To minimize the variance of $\hat\tau_\mathrm{DP-meta}$, one should choose weights that are inversely proportional to the variance of the original estimates.

\begin{proposition}
    For $\lambda^*_j = \frac{\left(\hat V_\mathrm{DP}^{(j)}/n_j\right)^{-1/2}}{\sum_{k=1}^N\left(\hat V_\mathrm{DP}^{(k)}/n_k\right)^{-1/2}}$, we have $\hat V_\mathrm{DP-meta} / n = \frac{1}{N}\left(\frac{1}{N}\sum_{j=1}^N\left(\hat V_\mathrm{DP}^{(j)}/n_j\right)^{-1/2}\right)^{-2}$.
\end{proposition}

As expected, this variance decreases as the number of studies increases, at a rate of $1/N$ when all variances are of the same order.
Assuming our estimators are used, and applying the asymptotic variance results from \Cref{thm:utility_general}, the optimal weights are, to a first-order approximation, $\lambda^*_j \propto \left(V^\star_j/n_j + \frac{C_j}{(\gdpmu_1^{(j)}(K_j-1))^2}\right)^{-1/2} + o(1/\sqrt n_j) + o(1/K_j)$, where $V^\star_j$ is the oracle variance, $K_j$ the number of folds, and $C_j$ the estimator-specific constant (see \Cref{thm:privacy_unified}) for study $j$.
Here, the classical oracle variance term appears—assigning smaller weights to studies with higher variance—augmented by an additional term reflecting the impact of privacy: studies with weaker privacy constraints receive larger weights in the meta-analysis.
%This makes appear both the classical and expected oracle variance term (the larger the variance of a center, the smaller the weight), as well as an additional term in the weight that is directly related to \textbf{privacy}: the easier it is for a center to be private, the larger its weight will be in the meta analysis.

\section{Experiments with synthetic data}
\label{app:exp}

\subsection{Experimental details}
\label{app:exp:setup}

We give details on the setup used in the experiments of Section~\ref{sec:DP_estimators}.

In \Cref{fig:experiments-plot} the subscript for our estimators denote the estimator for the nuisance functions. 
$IPW_{log}$ and $IPW_{tree}$ estimates the propensity score with a logistic model and decision tree, respectively. 
$G_{lin}$ and $G_{tree}$ estimate the potential outcomes respectively with a linear regression model or a decision tree.
$AIPW_{miss.}$ is misspecified for both nuisance functions and uses a logistic model for the propensity score and linear regression for the potential outcomes. 
$AIPW_{correct}$ is correctly specified by using a decision tree for both nuisance estimators.
$AIPW_{log}$ and $AIPW_{lin}$ are partially misspecified by using a logistic model for the propensity score or a linear regression for the potential outcomes, respectively.
In \Cref{fig:experiment1-well-specified} shown later in this section, $AIPW_{log}$ uses a logistic model both for the propensity score and the potential outcomes.

Unless otherwise specified, we clip real-value outcomes to $[-1,1]$ and use parameter $\Bmu = 1$. For all settings most data points are not clipped.

\paragraph{Well-specified setting with low overlap.} We generate data with $n = 5000$ and $d = 1$, where the propensity score is defined as $\pi(X_i) = \textrm{clip}_{[0.004,0.996]}(\mathrm{expit}(-0.2 + 6 X_{i,1}))$, where $\mathrm{expit}(x) \eqdef 1/(1+\exp(-x))$.
The observed outcome follows a linear regression model where $Y_i = -0.05 + 0.225 X_{i,1} + 0.1A_i + e$, with $e \sim \mathcal{N}(0,0.01)$.

%For this setting both the propensity score and the outcomes are generated as logistic regressions.
%We set $n=50000$, $d=10$ and $\eta = 0.1$. 
%The propensity score is set as 
%$\pi(x) = \mathrm{expit}(0.1 - 0.15 X_{i1} + 0.225X_{i2} - 0.15X_{i3} - 0.2X_{i4} + 0.1X_{i5} + 0.05X_{i6} - 0.075X_{i7} + 0.225X_{i8} - 0.15X_{i9} - 0.2X_{i10})$.
%$\pi(x) = \textrm{clip}_{[0.1,0.9]}(\mathrm{expit}(0.1 + X_{i}^T\beta))$, where $\beta = (-0.15, 0.225, -0.15, -0.2, 0.1, 0.05, -0.075, 0.225, -0.15, -0.2)$ is chosen to give a range of propensity scores where very few samples require clipping.
%The outcome of a sample is generated by $\mathrm{Bernoulli(\mathrm{expit}(-0.05 + Y_i^T\gamma + 0.42585A_i))}$, where $\gamma = (0.175, 0.1, -0.125, 0.075, -0.1, 0.2, -0.2, 0.175, -0.1, 0.2)$. This gives a modest $ATE$ of $0.1$.
%For this setup we would expect the previous techniques to perform very well. It is an ideal setup for those technique that are both designed for logistic regression propensity scores. 
%Propensity score $\pi(x)= \mathrm{expit}(\beta X_{i1})$ for $\beta >0$.
%Outcomes $Y_i\sim \cN(0,1) + \zeta A_i (X_{i1}+X_{i2})$ for $\zeta >0$.

\paragraph{Misspecified setting.}
Here, propensity scores and outcome responses do not follow linear models but are instead based on splitting the feature space into four different regions for both propensities and outcome responses. These regions are modeled by two different decision trees with overlapping regions, where samples that are more likely to receive treatment have worse outcomes---a setup that could occur in medical settings, where practitioners might prescribe particular drugs based on 2 symptoms.
We consider a simple setup with $d=2$, $n = 250000$, $\Bpi = 1/0.2$, $\Bmu = 1$.
The outcome is set as $Y_i = \gamma(X_i) + 0.2 A_i + e$, where $e \sim \mathcal{N}(0, 0.025)$.
The propensity score and $\gamma(X_i)$ are as follows: 

\begin{minipage}{0.475\linewidth}
\[
\pi(X_i) = 
\begin{cases}
   0.75 & \text{if } X_{i,1} > 0.1 \text{ and } X_{i,2} > 0.\\
   0.6 & \text{if } X_{i,1} \leq 0.1 \text{ and } X_{i,2} > 0.\\
   0.25 & \text{if } X_{i,1} < -0.05 \text{ and } X_{i,2} \leq 0.\\
   0.5 & \text{if } X_{i,1} \geq -0.05 \text{ and } X_{i,2} \leq 0.
\end{cases}
\]
\end{minipage}
\hfill
\begin{minipage}{0.475\linewidth}
\[
\gamma(X_i) = 
\begin{cases}
   -0.7 & \text{if } X_{i,1} > 0 \text{ and } X_{i,2} > 0.\\
   0.1 & \text{if } X_{i,1} > 0 \text{ and } X_{i,2} \leq 0.\\
   -0.4 & \text{if } X_{i,1} \leq 0 \text{ and } X_{i,2} > 0.05.\\
   0.6 & \text{if } X_{i,1} \leq 0 \text{ and } X_{i,2} \leq 0.05.
\end{cases}
\]
\end{minipage}

\paragraph{Comparison with non-private G-Formula.}
For the experiment in \Cref{fig:sample-AZH} (left) both the propensity score and the outcome follow a logistic regression model. 
We use the same parameters as in \Cref{app:exp:results} when we compare with prior methods in a well-specified setting.
All estimators use G-Formula with a logistic regression estimate for the nuisance functions.
For the non-private baseline we use cross-fitting by splitting the data into two folds. 
We set $K=n/100$ for all our estimators.

\subsection{Additional results}
\label{app:exp:results}

% Neurips checklist

\paragraph{Well-specified setting with good overlap.}
For this experiment, we again consider a well-specified setting for \cite{lee2019privacyIPW, ohnishi2024covariateBalancing, DBLP:journals/csda/GuhaR25}.
The outcomes $Y_i \in \{0,1\}$ are binary, and both the propensity score and potential outcomes follow logistic regression models.
This setup is essentially the same as one of the well-specified experiments in \cite{ohnishi2024covariateBalancing}. 
%For our estimators, we run IPW and AIPW with well-specified nuisance estimators, with $K=500$.
We set $n=50000$, $d=10$ and $\Bpi = 1/0.1$. 
The propensity score is $\pi(x) = \textrm{clip}_{[0.1,0.9]}(\mathrm{expit}(0.1 + X_{i}^T\beta_\pi))$, where $\beta_\pi = (-0.15, 0.225, -0.15, -0.2, 0.1, 0.05, -0.075, 0.225, -0.15, -0.2)$. 
Here $\beta_\pi$ is chosen to give us a range of propensity scores while only few samples require clipping. That is, we have large overlap for most samples.
The outcome of a sample is generated by $\mathrm{Bernoulli(\mathrm{expit}(-0.05 + X_i^T\beta_\mu + 0.42585A_i))}$, where $\beta_\mu = (0.175, 0.1, -0.125, 0.075, -0.1, 0.2, -0.2, 0.175, -0.1, 0.2)$. This gives an $ATE$ of $0.1$.

For our estimators, we set $K = 500$ and run IPW, G-Formula and AIPW with logistic regression for the nuisance estimators.  
We also run G-Formula using a simple linear regression model.
We run the experiment with two sets of privacy parameters.
This experiment is an ideal setup for previous work, so we expect them to perform well here. This is confirmed by the results in 
\Cref{fig:experiment1-well-specified}.
However, our G-Formula estimator achieves similar performance and outperforms our other estimators, as it requires less noise thanks to its lower sensitivity. 
% We also run a well-specified G-Formula estimator with logistic regression model as well as a linear regression model. 
% Finally, we run a more flexible model using a random forest classifier. 

% As expected, we see that the related work perform well in this experiment with \cite{ohnishi2024covariateBalancing} outperforming the other techniques. 
% However, using G-Formula we are still are to get similar performance.

%\ab{remove the random forest classifier from this figure, it comes a bit out of the blue I think}

\begin{figure}[htbp]
  \centering
 %  \begin{subfigure}[t]{0.48\textwidth}
    \centering
    \includegraphics[width=\linewidth]{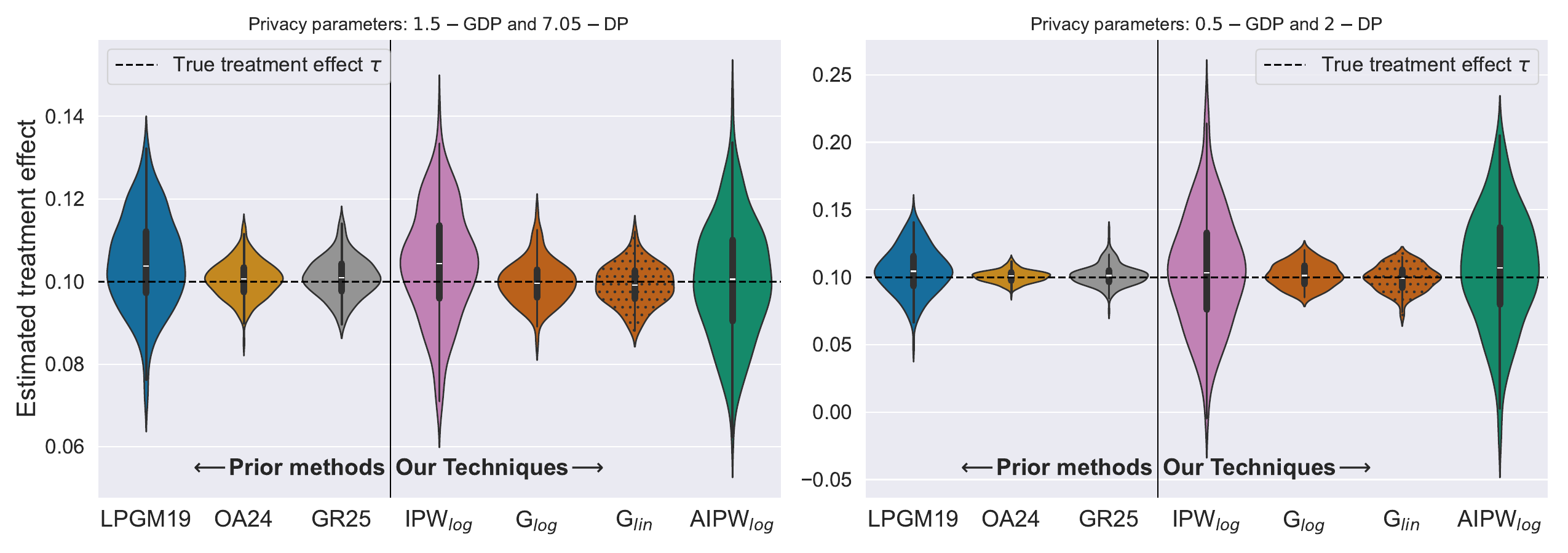}
    %\caption{Box plot of distributions.}
 %    \label{fig:plot1_low_privacy}
  % \end{subfigure}
  % \hfill
  % \begin{subfigure}[t]{0.48\textwidth}
%     \centering
 %    \includegraphics[width=\linewidth]{images/plot2_violin.pdf}
    %\caption{Violin plot showing full distributions.}
 %    \label{fig:plot1_high_privacy}
 %  \end{subfigure}
  \caption{Well-specified setting with good overlap: propensity scores and outcome are linear models.}
  \label{fig:experiment1-well-specified}
\end{figure}

%\textbf{Effect of parameter $K$.}\label{app:exp:effect-k}
\subsection{Effect of parameter $K$}\label{app:exp:effect-k}
Our approach requires selecting the number of folds $K$. By increasing $K$, we reduce sensitivity, allowing us to achieve differential privacy with less noise as shown in \Cref{sec:DP_estimators}.
However, if we set $K$ too high, the nuisance estimators may not have enough samples to train well. 
Therefore, we want to use the largest value of $K$ that maintains good nuisance estimation performance.
This choice depends on both the choice of nuisance estimator and the dataset size.
Here, we evaluate the effect of $K$ on our estimators for datasets with $n=20000$, $d = 20$, and $\Bpi = 1/0.1$.
The propensity score follows a logistic regression model where $\pi(X_i) = \textrm{clip}_{[0.1,0.9]}(\mathrm{expit}(0.1 + X_{i}^T\beta_\pi))$ with $\beta_\pi =  (-0.17, \allowbreak -0.06, \allowbreak 0.05, \allowbreak 0.14, \allowbreak 0.12, \allowbreak -0.195, \allowbreak -0.205, \allowbreak 0.07, \allowbreak 0.18, \allowbreak 0.14, \allowbreak -0.14, \allowbreak 0.05, \allowbreak 0.01, \allowbreak -0.16, \allowbreak -0.18, \allowbreak -0.1, \allowbreak 0.2, \allowbreak 0.03, \allowbreak -0.16, \allowbreak -0.1)$.
The outcome follows a linear regression model such that $Y_i = -0.08 + X_{i}^T\beta_Y + 0.15 A_i + e$, where $e \sim \mathcal{N}(0, 0.0025)$ and $\beta_Y = (-0.0385, \allowbreak -0.0111, \allowbreak -0.105, \allowbreak -0.0344,  \allowbreak 0.1405, \allowbreak 0.0550, \allowbreak 0.0344, \allowbreak -0.0908, \allowbreak -0.0023, \allowbreak -0.0243, \allowbreak -0.0076, \allowbreak -0.0416, \allowbreak 0.0193, \allowbreak -0.0846, \allowbreak 0.0582, \allowbreak 0.0824, \allowbreak 0.0184, \allowbreak 0.0064, \allowbreak -0.0895, \allowbreak 0.0241)$.
We do not add Gaussian noise in these experiments. The purpose is to evaluate the effect of the parameter $K$ on the aggregated estimate.

The results are shown in \Cref{fig:experiment3-G-formula} (G-Formula) and \Cref{fig:experiment3-IPW-AIPW} (IPW and AIPW). %\Cref{fig:experiment3-IPW} (IPW) and \Cref{fig:experiment3-AIPW} (AIPW).

\begin{figure}[htbp]
  \centering
  \includegraphics[width=\linewidth]{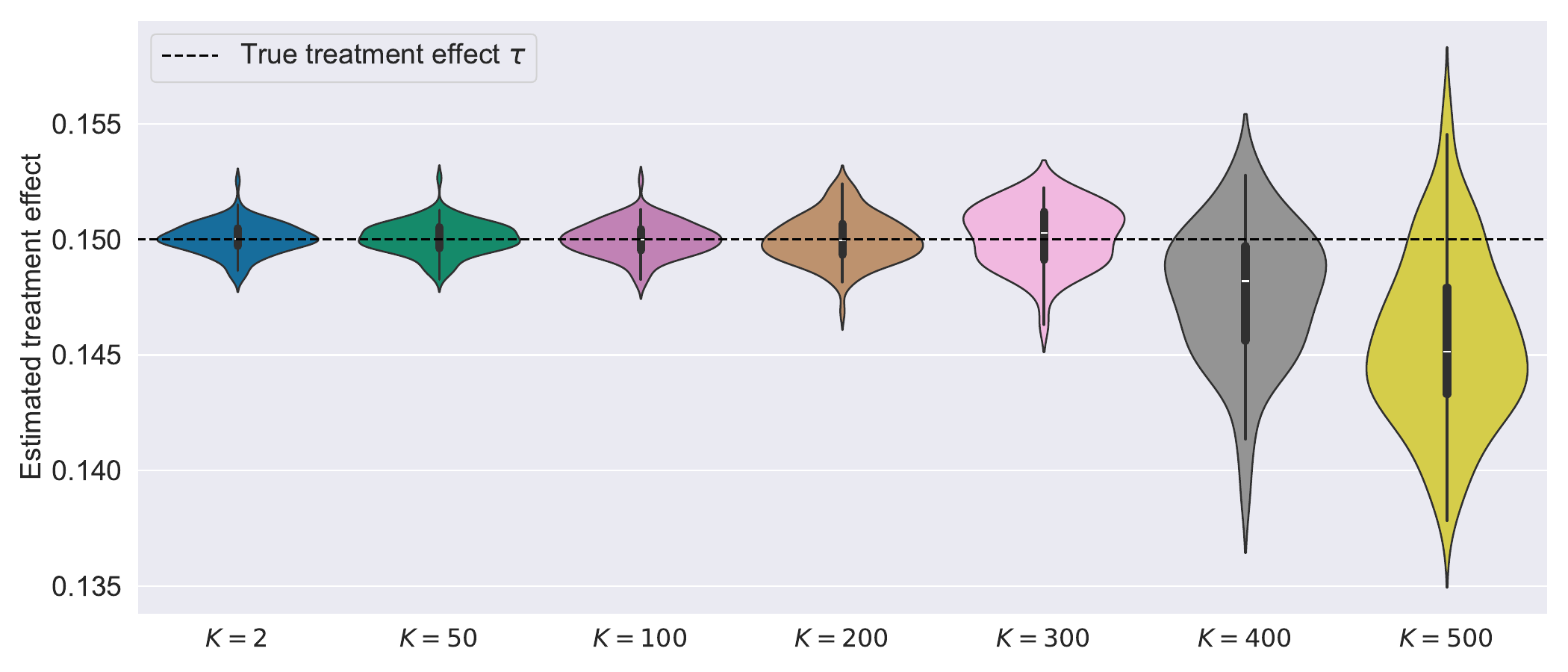}
%\caption{Box plot of distributions.}
  \caption{Effect of changing $K$ for G-Formula linear regression. The aggregated estimators perform well even as $K$ increases. Although the nuisance estimators are each trained on smaller folds which increases variance, this is balanced by aggregating more estimators.
  However, at $K=400$ the performance of the aggregator estimator degrades as each fold contains only $50$ samples.}
  \label{fig:experiment3-G-formula}
\end{figure}

\begin{figure}[htbp]
  \centering
  \begin{subfigure}[t]{0.48\textwidth}
    \centering
    \includegraphics[width=\linewidth]{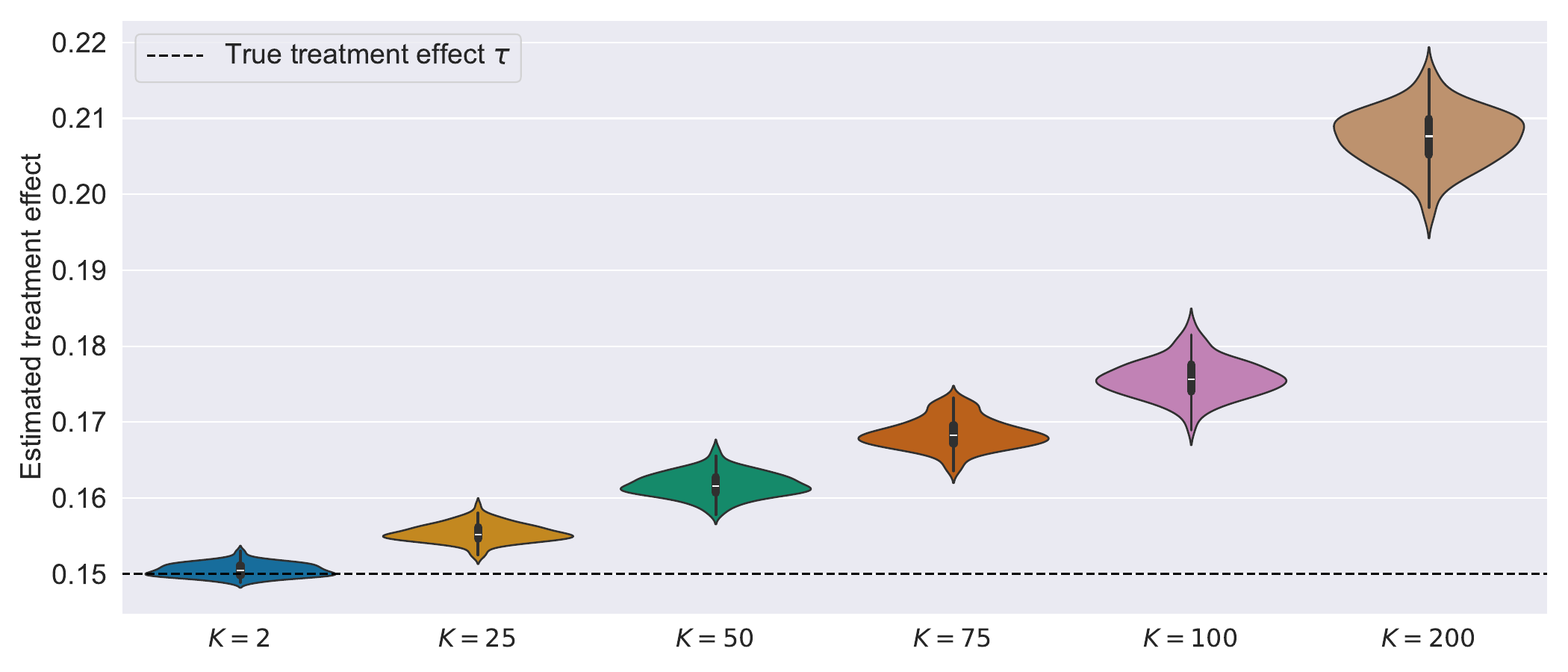}
    %\includegraphics[width=\linewidth]{images/plot_simple_wellspecified_violin.pdf}
    %\caption{Box plot of distributions.}
 %    \label{fig:plot1_low_privacy}
   \end{subfigure}
   \hfill
   \begin{subfigure}[t]{0.48\textwidth}
     \centering
     \includegraphics[width=\linewidth]{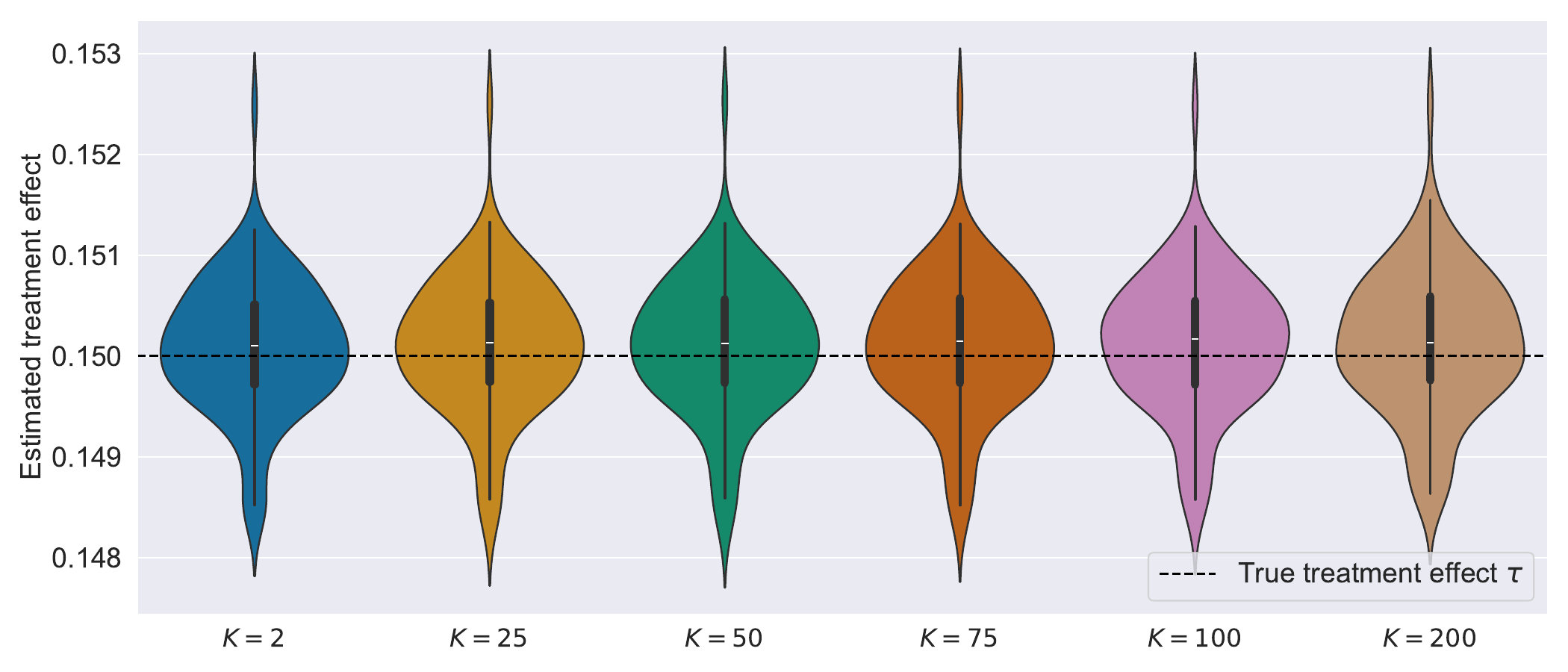}
    %\caption{Violin plot showing full distributions.}
  %   \label{fig:plot1_high_privacy}
   \end{subfigure}
   \caption{Effect of changing $K$ for IPW and AIPW. We see that in this setup the accuracy of IPW quickly decreases and the smaller fold size introduces bias. In contrast, AIPW benefits from the performance of G-Formula and remains stable for moderate values of $K$.}
   \label{fig:experiment3-IPW-AIPW}
\end{figure}

\paragraph{On nuisance function convergence rates}
Our theoretical result in \Cref{thm:utility_general} presents the asymptotic variance of our ATE estimators.
The results are conditioned on sufficiently fast convergence rates of the nuisance estimators. 
This is similar to the guarantees for the non-private setting as presented in \Cref{sec:prelim} that rely on the same convergence rate assumptions.
In the non-private setting, nuisance estimators are often trained on half the dataset, while we typically train on a larger number of smaller folds.
It is natural to wonder if the assumptions still hold especially if we set $K$ to a function of $n$ rather than a constant.
Since we train nuisance estimators on a smaller fraction of the dataset, we expect each non-private estimator to converge slower. However, we argue that this is offset by the ensemble technique. 

Consider a simplified example, where we want to apply $\hat\tau_\mathrm{DP-AIPW}$ and we have that $(\hat\pi^{(k)}(X_i)-\pi(X_i)) \sim \mathcal{N}(0,1/m)$ and $(\hat\mu^{(k)}_a(X_i)-\mu_a(X_i)) \sim \mathcal{N}(0,1/m)$, whenever $\hat\pi^{(k)}$ and $\hat\mu^{(k)}_a$ are trained on folds with $m$ data points.
If we train the nuisance estimators with standard cross-fitting (2 folds) we have $(\hat\pi(X_i)-\pi(X_i)) \sim \mathcal{N}(0,2/n)$ and $(\hat\mu_a(X_i)-\mu_a(X_i)) \sim \mathcal{N}(0,2/n)$ and thus $\esp{\left(\hat\pi(i)-\pi(X_i)\right)^2}\esp{\left(\hat\mu_a(i)-\mu_a(X_i)\right)^2} = 4/n^2 = o(n^{-1})$ as required.
Using our folding strategy the error of each nuisance estimator would now be distributed as $\mathcal{N}(0, K/n)$. However, we average estimators trained on independent datasets which lowers the variance by a factor $K - 1$ such that the aggregate estimate has error of $\mathcal{N}(0, K/(K-1) \cdot 1/n)$.
Thus we have that $\esp{\left(\hat\pi(i)-\pi(X_i)\right)^2}\esp{\left(\hat\mu_a(i)-\mu_a(X_i)\right)^2} = (K/(K-1))^2 \cdot 1/n^2 = o(n^{-1})$.

In this simplified example, we can thus set $K$ to be arbitrary high without degrading the estimates. However, in practice we cannot make the folds too small, as it can introduce bias. Unlike the increase in variance, bias will not degrade from aggregating estimators. This is backed up by our numerical observations: in \Cref{fig:experiment3-IPW-AIPW} (right), the estimates remain stable as we increase $K$; but when the folds are too small, such as the right-most plots in \Cref{fig:experiment3-G-formula}, bias kicks in and degrades the performance. 

%$\esp{\left(\hat\pi(i)-\pi(X_i)\right)^2}\esp{\left(\hat\mu_a(i)-\mu_a(X_i)\right)^2}=o(n^{-1})$ if $\hat\tau_\mathrm{DP}
%\clin{TODO: Finish example - possible AIPW is the best case, because the convergence rate is more realistic there.}

%The convergence rate for nuisance estimations is indeed slower when trained on a fold instead of the full dataset. Thus, the assumptions of the theorem are indeed stronger than the ones that are usually made, for which  (or , with cross-fitting). However, note that we are aggregating estimators, so the assumption might hold for the aggregate nuisance estimator even if the convergence rate is slower for each individual estimator, we give a simple example below. As long as these assumptions are satisfied, aggregating all nuisance estimators on all samples thus leads to improved convergence rates and limited sample variance.

Note that if the actual convergence rate of nuisance estimators is too slow, the technique from \Cref{app:CI_var} might report too small confidence intervals. We stress the fact that this is a challenge also in the non-private setting.
In cases where this assumption on convergence rate does not hold, confidence intervals can instead be reported as in \Cref{sec:private_CIs}, via private bootstrap techniques or built-in causal forests methods, which take into account the variability of both nuisance training and sample variance. 

%\clin{The last reviewer was concerned with convergence rate. Expand with toy example in appendix? We should also hightlight the section from 4.3: "Nevertheless, we go beyond relying on these assumptions with our techniques for reporting confidence intervals discussed in Section 4.3 with further details in Appendix B. We show that our framework can accommodate techniques for computing confidence intervals such as pointwise variance or the bootstrap. These techniques can be used for computing valid CIs when convergence rate assumptions from Theorem 2 do not hold, and it is how people most often report CIs in practice. We acknowledge that this setting should be further discussed in the main body, and we will update the introduction to highlight this contribution."}

\subsection{Comparison to subsample and aggregate}
\label{app:subsample-and-aggregate}

%\clin{We shouldn't compare to this for this paper, but there is an interesting new approach that improves SA. It might be worth looking into for future work \url{https://bsky.app/profile/stein.ke/post/3m2exluzeic2r}. It's unclear if it is applicable to this setting or computationally efficient enough.}

In the Subsample and Aggregate (SA) framework~\citep{NissimRS07,smith11}, the data is first split into disjoint folds.
Each fold is used to compute some estimate independently of the other folds.
These estimates are then aggregated in the differentially private manner. 
Our folding scheme is conceptually similar to the SA framework: 
just as in our estimators, increasing the number of folds reduces the sensitivity since each sample only influences one of the estimates.
However, a key difference is that, in our approach, nuisance estimators are combined across folds, thus increasing the effective sample size for nuisance function estimates of each score. 
Additionally, the subsample and aggregate framework has to further divide folds to ensure data used for nuisance estimation and ATE estimation is independent to avoid introducing bias, whereas we achieve similar guarantees by combining folds. 

Here, we evaluate the SA framework when adapted to our setting and estimators. 
% For any fold $\cI_k$, we need to perform cross-fitting to avoid introducing bias. 
For any fold $\cI_k$, we split each fold into two smaller folds $\cI_k^{(1)}$ and $\cI_k^{(2)}$, learn nuisance functions on one fold, and compute score functions using these nuisance functions on data from the other fold.
That is, if a sample $i\in[n]$ is in $\cI_k^{(1)}$, we have $\Gamma_i^\mathrm{G} = \mu_1^{(k,2)}(X_i) - \mu_0^{(k,2)}(X_i)$, where $\mu_0^{(k,2)},\mu_1^{(k,2)}$ are learned on $\cI_k^{(2)}$. 
Similarly, we use nuisance functions learned on $\cI_k^{(1)}$ to compute scores for samples in $\cI_k^{(2)}$.

In \Cref{fig:subsample-and-aggregate}, we compare effect of changing the parameter $K$ for the two folding schemes. 
For simplicity, we focus on reporting an ATE estimate with the G-Formula, we use the same synthetic data generation that we use for varying the parameter $K$ in \Cref{app:exp:effect-k}, and we do not add noise in order to isolate the effect of the folding schemes.
% Just as in our scheme, the estimate is relatively stable  when $K$ is not too large.
As expected, the plot shows that our folding scheme consistently performs better and remains more stable for larger values of $K$.

\begin{figure}[htbp]
  \centering
   \includegraphics[width=\linewidth]{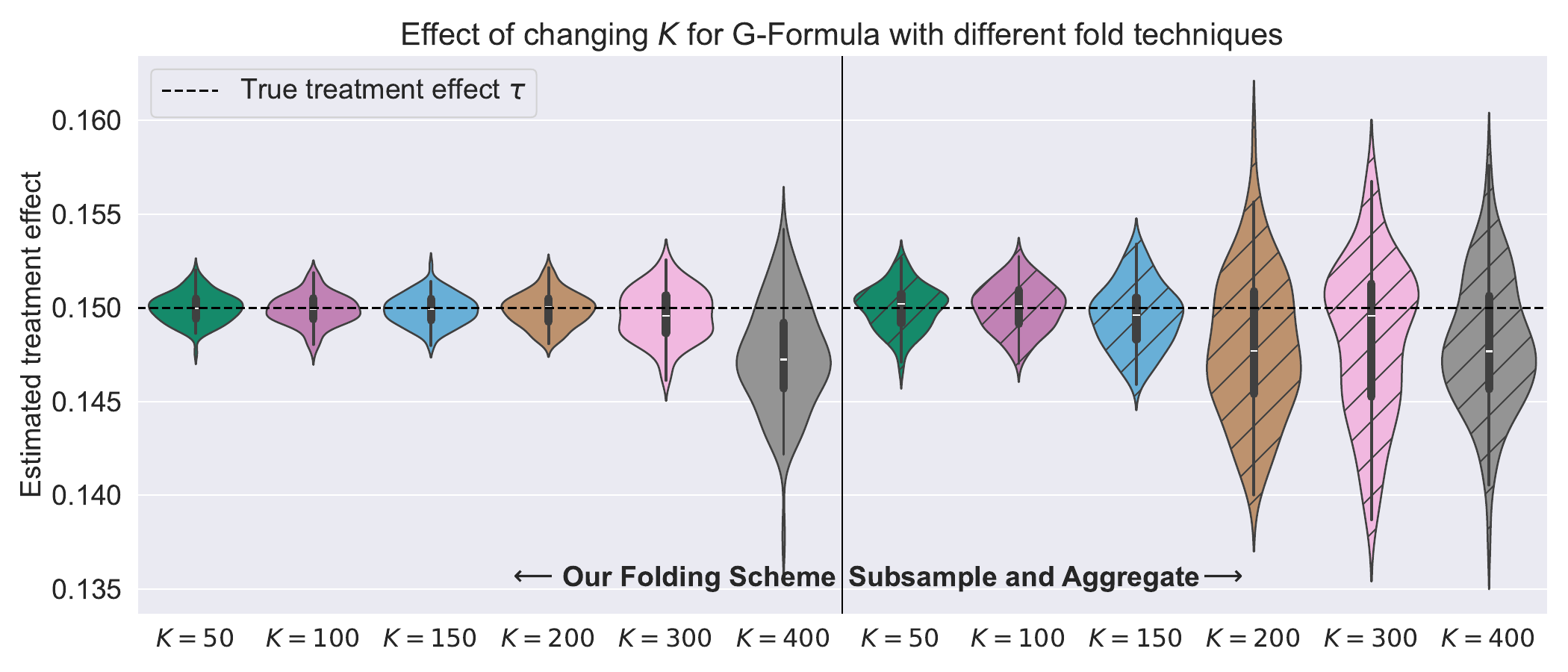}
   \caption{Effect of changing $K$ for G-Formula using our folding technique (left) and the Subsample and Aggregate folding technique (right). The setting is identical to \Cref{fig:experiment3-G-formula}.}
   \label{fig:subsample-and-aggregate}
\end{figure}

\section{Real datasets experimental details}
\label{app:details_real_semi-synth}

The license information and terms of use of the ADNI and NACC datasets can be found respectively at
\url{https://adni.loni.usc.edu/terms-of-use/} and \url{https://files.alz.washington.edu/documentation/nacc_data_use_agreement.pdf}.

\paragraph{NACC preprocessing.}
The NACC Unified Data Set (UDS) \citep{beekly2007national} includes clinical, demographic, medication, and genetic information collected annually across 36 Alzheimer’s Disease Research Centers. For this study, data from 2005–2016 were used. Patients were selected if they had at least one visit with a diagnosis of mild cognitive impairment (MCI) due to AD, and those already prescribed cholinesterase inhibitors (ChEIs) at their first UDS visit were excluded. Treated patients were defined as those who began ChEI treatment during follow-up, while controls were those who never reported ChEI use. Patients who switched treatment during follow-up were censored at the switch date.

\paragraph{ADNI preprocessing.}
The ADNI dataset \citep{petersen2010alzheimer}, a large multi-center cohort, contains biomarker, imaging, and neuropsychological test data from participants aged 55–90, including normal, MCI, and AD individuals. As with NACC, patients were filtered to those meeting MCI due to AD criteria and aligned in terms of treatment history. Both datasets were harmonized by defining the baseline visit as the first MCI diagnosis and by excluding patients who had received prior treatment.

\paragraph{Merging and alignment.}
To ensure comparability between treated and untreated groups, the baseline for treated patients was set to the time of their first ChEI prescription. Controls were aligned by shifting their baseline to the median lag between MCI diagnosis and treatment initiation in the treated group. Only control patients with a visit within ±6 months of this shifted baseline (of 1 years) were retained.  Overlap trimming (with trimming of $0.05$) was then applied to exclude patients with extreme propensity scores.
This preprocessing reduced the combined sample from $\sim$ 12,000 to $\sim$ 5,200 patients, of whom 3,215 remained after filtering (712 treated) those who had a visit in the shifted baseline.

\paragraph{Covariates selected.}
For our analysis, we selected a set of covariates from the NACC database that capture demographic characteristics, cognitive performance, functional status, and medication use at baseline. Specifically, these include: treatment indicator $A$ and outcome $Y$, demographic variables such as age (AGE) and education (EDUC), cognitive assessment scores including Montreal Cognitive Assessment (MOCA), Mini-Mental State Examination (MMSE), Clinical Dementia Rating global score (CDRGLOB), and Clinical Dementia Rating sum of boxes (CDR); functional assessment measures such as Functional Activities Questionnaire (FAQ); depression score (Geriatric Depression Scale, GDS); the time in years from cognitive threshold (YEARS\_FROM\_COGNITIVE\_THRESHOLD); body-mass index (NACCBMI); and a summary of medications (MEDS). All scores were measured at baseline, ensuring that the covariates reflect the patient’s status prior to any outcome assessment or intervention.

\paragraph{Missing data.}
Missing data is imputed using iterative imputations with the MICE package \citep{mice} with its R by-default parameters, prior to all our analysis.  We thus assume that the dataset to privatize is the imputed dataset.
Bootstraps do not take into account the randomness that comes from imputation, and we assume that the dataset to privatize is the imputed one, setting that problem apart.
Handling missing data properly under DP involves further complications, but we ignore those challenges. E.g., if data is imputed with a mean, then one datapoint can affect the data of others, while using more sophisticated imputations (like mice) can create further problems. Interestingly, this is not an issue at all if imputation is handled separately for each fold.
However, missing data imputation is orthogonal to our work.

\paragraph{Outcome.}
The outcome we consider is the evolution of the MMSE score after baseline.
The outcome for patient $i$ is thus constructed as the difference in MMSE score between the 1 year visit ($\pm 6$ months) and the baseline visit.
The MMSE scores are on a scale of 30. A score of 23 or less is the generally accepted cutoff point indicating the presence of cognitive impairment. Levels of impairment have also been classified as none (24-30); mild (18-23) and severe (0-17).

\paragraph{Parameters of the algorithms.}
The number of bootstraps is set to $R=200$.
Regression forests and causal forests are taken with 300 trees, a minimal node size of 5 and a sample fraction of 0.4, to avoid empty leaves.
The outcomes are bounded between 0 and 9 after reversing signs and translating them, so that we can take $\Bmu=4.5$.

\end{document}